\documentclass{article}

\usepackage[preprint, nonatbib]{nips_2018}

\usepackage[numbers]{natbib}
\usepackage[utf8]{inputenc} % allow utf-8 input
\usepackage[T1]{fontenc}    % use 8-bit T1 fonts
\usepackage{hyperref}       % hyperlinks
\usepackage{url}            % simple URL typesetting
\usepackage{booktabs}       % professional-quality tables
\usepackage{amsfonts}       % blackboard math symbols
\usepackage{nicefrac}       % compact symbols for 1/2, etc.
\usepackage{microtype}      % microtypography
\usepackage{amsmath}
\usepackage{amsthm}
\usepackage{xcolor}
\usepackage{graphicx}
\usepackage{tabularx}
\usepackage{algpseudocode}
\usepackage{algorithm}
\usepackage{algorithmicx}
\usepackage{tikz}
\usepackage{enumerate}
\usepackage{float}
\usepackage{placeins}
\usepackage{authblk}

\usetikzlibrary{shadings}
\usetikzlibrary{positioning}
\usetikzlibrary{calc}
\usetikzlibrary{fit}
\usetikzlibrary{backgrounds}
\usetikzlibrary{shapes.geometric}
\usetikzlibrary{patterns,snakes}

\title{Observe and Look Further:\\Achieving Consistent Performance on Atari}

\author[1]{\textbf{Tobias Pohlen}}
\author[1]{\textbf{Bilal Piot}}
\author[1]{\textbf{Todd Hester}}
\author[1]{\textbf{Mohammad Gheshlaghi Azar}}
\author[1]{\textbf{Dan Horgan}}
\author[1]{\textbf{David Budden}}
\author[1]{\textbf{Gabriel Barth-Maron}}
\author[1]{\textbf{Hado van Hasselt}}
\author[1]{\textbf{John Quan}}
\author[1]{\textbf{Mel Večerík}}
\author[1]{\textbf{Matteo Hessel}}
\author[1]{\textbf{Rémi Munos}}
\author[2]{\textbf{Olivier Pietquin}}
\affil[1]{DeepMind, \texttt{\{pohlen, piot, toddhester, mazar, horgan, budden, gabirelbm, hado, johnquan, vec, mtthss, munos\}@google.com}}
\affil[2]{Google Brain, \texttt{pietquin@google.com}}

\newtheorem{prop}{Proposition}[section]
\newtheorem{lemma}{Lemma}[section]

\newcommand{\ie}[0]{\emph{i.e.,}~}

\newcommand{\argmax}{\operatorname*{argmax}}
\newcommand{\argmin}{\operatorname*{argmin}}

\newcommand\RotText[1]{\rotatebox{90}{\centering#1}}
\DeclareRobustCommand\linereference[1]{\tikz[baseline=-0.5ex]{\draw[color=#1, line width=1pt] (0, 0) -- (0.5, 0);}}
\definecolor{dmgray}{RGB}{204,204,204}
\newcolumntype{C}{>{\centering\arraybackslash}X}
\newcolumntype{R}{>{\raggedleft{}\arraybackslash}X}
\newcolumntype{?}{!{\vrule width 1pt}}
\newcolumntype{-}{!{\color{dmgray}\vrule}}

\definecolor{dmlightgray}{RGB}{80, 80, 80}
\definecolor{dmblue}{rgb}{0.2980392156862745, 0.4470588235294118, 0.6901960784313725}
\definecolor{dmgreen}{rgb}{0.3333333333333333, 0.6588235294117647, 0.40784313725490196}
\definecolor{dmgreendark}{RGB}{69, 137, 85}
\definecolor{dmred}{rgb}{0.7686274509803922, 0.3058823529411765, 0.3215686274509804}
\definecolor{dmpurple}{rgb}{0.5058823529411764, 0.4470588235294118, 0.6980392156862745}
\definecolor{dmyellow}{rgb}{0.8, 0.7254901960784313, 0.4549019607843137}
\definecolor{dmcyan}{rgb}{0.39215686274509803, 0.7098039215686275, 0.803921568627451}
\tikzset{
	dmnode/.style = {},
	dmrect/.style = {draw, line width=1pt, rounded corners=2pt, rectangle, color=dmlightgray, text=black},
	dmmajor/.style = {minimum width=75pt},
	dmcirc/.style = {circle,},
	dmpath/.style = {draw, line width=1pt, rounded corners=2pt, color=dmlightgray},
	dmarrow/.style = {->, >=stealth},
	dmlabel/.style = {font=\normalsize},
	dmmajor/.style = {draw, dmrect, minimum height=2cm, minimum width=3.5cm},
	dmminor/.style = {draw, dmrect, line width=0.5pt, font=\footnotesize, minimum height=0.5cm},
}

\definecolor{ablation_dqfd_tc}{rgb}{0.3333333333333333, 0.6588235294117647, 0.40784313725490196}
\definecolor{ablation_dqfd_transform}{rgb}{0.39215686274509803, 0.7098039215686275, 0.803921568627451}
\definecolor{ablation_dqn_tc}{rgb}{0.8, 0.7254901960784313, 0.4549019607843137}
\definecolor{ablation_dqfd}{rgb}{0.2980392156862745, 0.4470588235294118, 0.6901960784313725}
\definecolor{ablation_dqn}{rgb}{0.7686274509803922, 0.3058823529411765, 0.3215686274509804}
\definecolor{ablation_dqfd_expert_loss}{rgb}{0.5058823529411764, 0.4470588235294118, 0.6980392156862745}

\newcolumntype{P}[1]{>{\centering\arraybackslash}p{#1}}

\begin{document}
% \nipsfinalcopy is no longer used

\maketitle
%\vspace{-8mm}
%\begin{center}
%Correspondences: \texttt{pohlen@google.com}
%\end{center}
%\vspace{8mm}
\vspace{-2mm}
\begin{abstract}
Despite significant advances in the field of deep Reinforcement Learning (RL), today's algorithms still fail to learn human-level policies consistently over a set of diverse tasks such as Atari 2600 games. 
We identify three key challenges that any algorithm needs to master in order to perform well on all games: 
processing diverse reward distributions, reasoning over long time horizons, and exploring efficiently. 
In this paper, we propose an algorithm that addresses each of these challenges and is able to learn human-level policies on nearly all Atari games. 
A new transformed Bellman operator allows our algorithm to process rewards of varying densities and scales; an auxiliary \emph{temporal consistency loss} allows us to train stably using a discount factor of $\gamma = 0.999$ (instead of $\gamma = 0.99$) extending the effective planning horizon by an order of magnitude; and we ease the exploration problem by using human demonstrations that guide the agent towards rewarding states. 
When tested on a set of 42 Atari games, our algorithm exceeds the performance of an average human on 40 games using a common set of hyper parameters.
Furthermore, it is the first deep RL algorithm to solve the first level of \textsc{Montezuma's Revenge}.
\end{abstract}

\section{Introduction}
In recent years, significant advances in the field of deep Reinforcement Learning (RL) have led to artificial agents that are able to reach human-level control on a wide array of tasks such as some Atari 2600 games~\cite{bellemare2015arcade}.
In many of the Atari games, these agents learn control policies that far exceed the capabilities of an average human player~\cite{gruslys2018the, hessel2017rainbow, horgan2018distributed}.
However, learning human-level policies consistently across the entire set of games remains an open problem.

We argue that an algorithm needs to overcome three key challenges in order to perform well on all Atari games.
The first challenge is processing diverse reward distributions. 
An algorithm must learn stably regardless of reward density and scale. 
\citet{mnih2015human} showed that clipping rewards to the canonical interval $[-1, 1]$ is one way to achieve stability. 
However, this clipping operation may change the set of optimal policies. 
For example, the agent no longer differentiates between striking a single pin or all ten pins in \textsc{Bowling}. 
Hence, optimizing the unaltered reward signal in a stable manner is crucial to achieving consistent performance across games.
The second challenge is reasoning over long time horizons, which means the algorithm should be able to choose actions in anticipation of rewards that might be far away.
For example, in \textsc{Montezuma's Revenge}, individual rewards might be separated by several hundred time steps.
In the standard $\gamma$-discounted RL setting, this means the  algorithm should be able to handle discount factors close to 1.
The third and final challenge is efficient exploration of the MDP. An algorithm that explores efficiently is able to discover long trajectories with a high cumulative reward in a reasonable amount of time even if individual rewards are very sparse.
While each problem has been partially addressed in the literature, none of the existing deep RL algorithms have been able to address these three challenges at once.

In this paper, we propose a new Deep Q-Network (DQN)~\cite{mnih2015human} style algorithm that specifically addresses these three challenges. 
In order to learn stably independent of the reward distribution, we use a transformed Bellman operator that reduces the variance of the action-value function.
Learning with the transformed operator allows us to process the unaltered environment rewards regardless of scale and density.
We prove that the optimal policy does not change in deterministic MDPs and show that under certain assumptions the operator is a contraction in stochastic MDPs (\ie the algorithm converges to a fixed point) (see Sec.~\ref{sec:transform}).
Our algorithm learns stably even at high discount factors due to an auxiliary \emph{temporal consistency (TC) loss}.
This loss prevents the network from prematurely generalizing to unseen states (Sec.~\ref{sec:tcloss}) allowing us to use a discount factor as high as $\gamma = 0.999$ in practice.
This extends the effective planning horizon of our algorithm by one order of magnitude when compared to other deep RL approaches on Atari.
Finally, we improve the efficiency of DQN's default exploration scheme by combining the distributed experience replay approach of \citet{horgan2018distributed} with the Deep Q-learning from Demonstrations (DQfD) algorithm of \citet{hester2017learning}. 
The resulting architecture is a distributed actor-learner system that combines offline expert demonstrations with online agent experiences (Sec.~\ref{sec:apexdqfd}).

We experimentally evaluate our algorithm on a set of 42 games for which we have demonstrations from an expert human player
(see Table~\ref{tbl:demos}). 
% (see Table~E.1). % Submissions
Using the same hyper parameters on all games, our algorithm exceeds the performance of an average human player on 40 games, the expert player on 34 games, and state-of-the-art agents on at least 28 games.
Furthermore, we significantly advance the state-of-the-art on sparse reward games.
Our algorithm is the first to complete the first level of \textsc{Montezuma's Revenge} and it achieves a new top score of 3997 points on \textsc{Pitfall!} without compromising performance on dense reward games and while only using 5 demonstration trajectories.

\vspace{-1mm}
\section{Related work}
\vspace{-1mm}
\label{sec:RelatedWork}
\textbf{Reinforcement Learning with Expert Demonstrations (RLED):}
RLED seeks to use expert demonstrations to guide the exploration process in difficult RL problems. Some early works in this area~\cite{atkeson1997robot, schaal1997learning} used expert demonstrations to find a good initial policy before fine-tuning it with RL. More recent approaches have explicitly combined expert demonstrations with RL data during the learning of the policy or action-value function~\cite{chemali2015direct, kim2013learning,piot2014boosted}.
In these works, expert demonstrations were used to build an imitation loss function (classification-based loss) or max-margin constraints.
While these algorithms worked reasonably well in small problems, they relied on handcrafted features to describe states and were not applied to large MDPs. In contrast, approaches using deep neural networks allow RLED to be explored in more challenging RL tasks such as Atari or robotics. In particular, our work builds upon DQfD~\cite{hester2017learning}, which used a separate replay buffer for expert demonstrations, and minimized the sum of a temporal difference loss and a supervised classification loss.
Another similar approach is Replay Buffer Spiking (RBS)~\cite{lipton2017bbq}, wherein the experience replay buffer is initialized with demonstration data, but this data is not kept for the full duration of the training.
In robotics tasks, similar techniques have been combined with other improvements to successfully solve difficult exploration problems~\cite{nair2017overcoming, vevcerik2017leveraging}.

\textbf{Deep Q-Networks (DQN):} DQN~\cite{mnih2015human} used deep neural networks as function approximators to apply RL to Atari games.
Since that work, many extensions that significantly improve the algorithm's performance have been developed. For example, DQN uses a replay buffer to store off-policy experiences and the algorithm learns by sampling batches uniformly from the replay buffer; instead of using uniform samples, \citet{schaul2015prioritized} proposed prioritized sampling where transitions are weighted by their absolute temporal difference error.
This concept was further improved by Ape-X DQN~\cite{horgan2018distributed} which decoupled the data collection and the learning processes by having many actors feed data to a central prioritized replay buffer that an independent learner can sample from.

\citet{durugkar2017constrained} observed that due to over-generalization in DQN, updates to the value of the current state also have an adverse effect on the values of the next state.
This can lead to unstable learning when the discount factor is high. 
To counteract this effect, they constrained the TD update to be orthogonal to the direction of maximum change of the next state.
However, their approach only worked on toy domains such as Cart-Pole.
Finally, \citet{vanhasselt2016popart} successfully extended DQN to 
process unclipped rewards with an algorithm called PopArt, which adaptively rescales the targets for the value network to have zero mean and unit variance.

\section{Algorithm}
\label{sec:Algorithm}
In this section, we describe our algorithm, which consists of three components: (1)~The transformed Bellman operator; (2)~The temporal consistency (TC) loss; (3)~Combining Ape-X DQN and DQfD.

\subsection{DQN Background}
Let $\langle \mathcal{X}, \mathcal{A}, R, P, \gamma \rangle$ be a finite, discrete-time MDP where $\mathcal{X}$ is the state space, $\mathcal{A}$ the action space, $R$ the reward function which represents the one-step reward distribution $R(x,a)$ of doing action $a$ in state $x$, $\gamma\in [0,1]$ the discount factor and $P$ a stochastic kernel modelling the one-step Markovian dynamics ($P(x'|x,a)$ is the probability of transitioning to state $x'$ by choosing action $a$ in state $x$).
The quality of a policy $\pi$ is determined by the action-value function 
\begin{align*}
Q^{\pi} : \mathcal{X} \times \mathcal{A} \to \mathbb{R}, (x, a) \mapsto \mathbb{E}^{\pi}\left[\sum_{t \geq 0} \gamma^t R(x_t, a_t) \ | \ x_0 = x, a_0 = a\right],
\end{align*}
where $\mathbb{E}^\pi$ is the expectation over the distribution of the admissible trajectories $(x_0, a_0, x_1, a_1, \dots)$ obtained by executing the policy $\pi$ starting from state $x$ and taking action $a$. 
The goal is to find a policy $\pi^*$ that maximizes the state-value $V^\pi(x) := \max_{a \in \mathcal{A}} Q^\pi(x, a)$ for all states $x$, \ie find $\pi^*$ such that V$^{\pi^*}(x) \geq \sup_\pi V^\pi(x)$ for all $x \in \mathcal{X}$.
While there may be several optimal policies, they all share a common optimal action-value function $Q^*$~\cite{puterman1994markov}.
Furthermore, acting greedily with respect to the optimal action-value function $Q^*$ yields an optimal policy.
In addition, $Q^*$ is the unique fixed point of the \emph{Bellman optimality operator} $\mathcal{T}$ defined as
\begin{align*}
(\mathcal{T}Q)(x, a) := \mathbb{E}_{x' \sim P(\cdot | x, a)}\left[R(x, a) + \gamma 
\max_{a' \in \mathcal{A}} Q(x', a')\right], \quad \forall (x, a) \in \mathcal{X} \times \mathcal{A}
\end{align*}
for any $Q : \mathcal{X} \times \mathcal{A} \to \mathbb{R}$.
Because $\mathcal{T}$ is a $\gamma$-contraction, we can learn $Q^*$ using a fixed point iteration.
Starting with an arbitrary function $Q^{(0)}$ and then iterating $Q^{(k)} := \mathcal{T}Q^{(k - 1)}$ for $k\in\mathbb{N}$ generates a sequence of functions that converges to $Q^*$.

DQN~\cite{mnih2015human} is an online-RL algorithm using a deep neural network $f_\theta$ with parameters $\theta$ as a function approximator of the optimal action-value function $Q^*$.
The algorithm starts with a random initialization of the network weights $\theta^{(0)}$ and then iterates
\begin{align}
\theta^{(k)} := \argmin_\theta \mathbb{E}_{x, a}\left[\mathcal{L}(f_\theta(x, a) - (\mathcal{T}f_{\theta^{(k - 1)}})(x, a))\right], \label{eq:dqn}
\end{align}
where the expectation is taken with respect to a random sample of states and actions and $\mathcal{L}$ is the Huber loss~\cite{huber1964robust} defined as
\begin{align*}
\mathcal{L} : \mathbb{R} \to \mathbb{R}, x \mapsto \left\{ \begin{array}{ll} \frac{1}{2}x^2 & \text{if $|x| \leq 1$} \\ |x| - \frac{1}{2} & \text{otherwise} \end{array} \right.
\end{align*}
In practice, the minimization problem in (\ref{eq:dqn}) is only approximately solved by performing a finite and fixed number of stochastic gradient descent (SGD) steps\footnote{\citet{mnih2015human} refer to the number of SGD iterations as \emph{target update period}.} and all expectations are approximated by sample averages.

\subsection{Transformed Bellman Operator}
\label{sec:transform}
\citet{mnih2015human} have empirically observed that the errors induced by the limited network capacity, the 
approximate finite-time solution to (\ref{eq:dqn}), and the stochasticity of the optimization problem can cause the algorithm to diverge if the variance of the optimization target $\mathcal{T}f_{\theta^{(k - 1)}}$ is too high.
In  order to reduce the variance, they clip the reward distribution to the interval $[-1, 1]$. While this achieves the desired goal of stabilizing the algorithm, it significantly changes the set of optimal policies.
For example, consider a simplified version of \textsc{Bowling} where an episode only consists of a single throw. 
If the original reward is the number of hit pins and the rewards were clipped, any policy that hits at least a single pin would be optimal under the clipped reward function.
Instead of reducing the magnitude of the rewards, we propose to focus on the action-value function instead. 
We use a function $h : \mathbb{R} \to \mathbb{R}$ that reduces the scale of the action-value function. Our new operator $\mathcal{T}_h$ is defined as
\begin{align*}
(\mathcal{T}_hQ)(x, a) := \mathbb{E}_{x' \sim P(\cdot | x, a)}\left[h\left( R(x, a) + \gamma \max_{a' \in \mathcal{A}} h^{-1}(Q(x', a')) \right)\right], \quad \forall (x, a) \in \mathcal{X} \times \mathcal{A}.
\end{align*}
\begin{prop}
\label{prop:trivial}
Let $Q^*$ be the fixed point of $\mathcal{T}$ and $Q : \mathcal{X} \times \mathcal{A} \to \mathbb{R}$, then
\begin{enumerate}[(i)]
\item If $h(z) = \alpha z$ for $\alpha > 0$, then $\mathcal{T}_h^k Q \xrightarrow{k \rightarrow \infty} h \circ Q^* = \alpha Q^*$.
\item If $h$ is strictly monotonically increasing and the MDP is deterministic (\ie $P(\cdot | x, a)$ and $R(x, a)$ are point measures for all $x, a \in \mathcal{X} \times \mathcal{A}$), then $\mathcal{T}_h^k Q \xrightarrow{k \rightarrow \infty} h \circ Q^*$.
\end{enumerate}
\end{prop}
\begin{proof}
(i) is equivalent to linearly scaling the reward $R$ by a constant $\alpha > 0$, which implies the proposition.
For (ii) let $Q^*$ be the fixed point of $\mathcal{T}$ and note that $h \circ Q^*  = h \circ \mathcal{T}Q^* = h \circ \mathcal{T}(h^{-1} \circ h \circ Q^*) = \mathcal{T}_h(h \circ Q^*)$ where the last equality only holds if the MDP is deterministic.
\end{proof}

Proposition~\ref{prop:trivial} shows that in the basic cases when either $h$ is linear or the MDP is deterministic, $\mathcal{T}_h$ has the unique fixed point $h \circ Q^*$.
Hence, if $h$ is an invertible contraction and we use $\mathcal{T}_h$ instead of $\mathcal{T}$ in the DQN algorithm, the variance of our optimization target decreases while still learning an optimal policy.
In our algorithm, we use $h : z \mapsto \text{sign}(z) (\sqrt{|z| + 1} - 1) + \varepsilon z$ with $\varepsilon = 10^{-2}$ where the additive regularization term $\varepsilon z$ ensures that $h^{-1}$ is Lipschitz continuous (see 
Proposition~\ref{prop:stochastic}).
% Proposition~A.1). %Submission
We chose this function because it has the desired effect of reducing the scale of the targets while being Liptschitz continuous and admitting a closed form inverse.

In practice, DQN minimizes the problem in (\ref{eq:dqn}) by sampling transitions of the form $t = (x, a, r', x')$ from a replay buffer where $x \in \mathcal{X}, a \sim \pi(\cdot | x), r' \sim R(x, a)$, and $x' \sim P(x, a)$.
Let $t_1,...,t_N$ be $N$ transitions from the buffer with normalized priorities $p_1,...,p_N$, then for $k \in \mathbb{N}$ the loss function in (\ref{eq:dqn}) using the operator $\mathcal{T}_h$ is approximated as
\begin{align*}
\mathbb{E}_{x, a}\left[\mathcal{L}(f_\theta(x, a) - (\mathcal{T}_hf_{\theta^{(k - 1)}})(x, a))\right] &\approx \sum_{i = 1}^N p_i \mathcal{L}\left(f_\theta(x_i, a_i) - h(r'_i + \gamma h^{-1}(f_{\theta^{(k - 1)}}(x_i', a_i')))\right)\\
&=: L_{\text{TD}}(\theta; (t_i)_{i = 1}^N, (p_i)_{i = 1}^N, \theta^{(k - 1)})
\end{align*}
where $a'_i := \argmax_{a \in \mathcal{A}} f_{\theta^{(k - 1)}}(x'_i, a)$ for DQN and $a'_i := \argmax_{a \in \mathcal{A}} f_\theta(x'_i, a)$ for Double DQN~\cite{van2016deep}.

\subsection{Temporal consistency (TC) loss}
\label{sec:tcloss}
The stability of DQN, which minimizes the TD-loss $L_\text{TD}$, is primarily determined by the target $\mathcal{T}_hf_{\theta^{(k - 1)}}$.
While the transformed Bellman operator provides an atemporal reduction of the target's scale and variance, instability can still occur as the discount factor $\gamma$ approaches 1.
Increasing the discount factor decreases the temporal difference in value between non-rewarding states.
In particular, unwanted generalization of the neural network $f_\theta$ to the next state $x'$ (due to the similarity of temporally adjacent target values) can result in catastrophic TD backups.
We resolve the problem by adding an auxiliary \emph{temporal consistency (TC) loss} of the form
\begin{align*}
L_{\text{TC}}(\theta; (t_i)_{i = 1}^N, (p_i)_{i = 1}^N, \theta^{(k - 1)}) := \sum_{i = 1}^N p_i \mathcal{L}\left(f_\theta(x'_i, a_i') - f_{\theta^{(k - 1)}}(x_i', a_i') \right)
\end{align*}
where $k \in \mathbb{N}$ is the current iteration.
The TC-loss penalizes weight updates that change the next action-value estimate $f_\theta(x', a')$.
This makes sure that the updated estimates adhere to the operator $\mathcal{T}_h$ and thus are consistent over time.

\subsection{Ape-X DQfD}
\label{sec:apexdqfd}
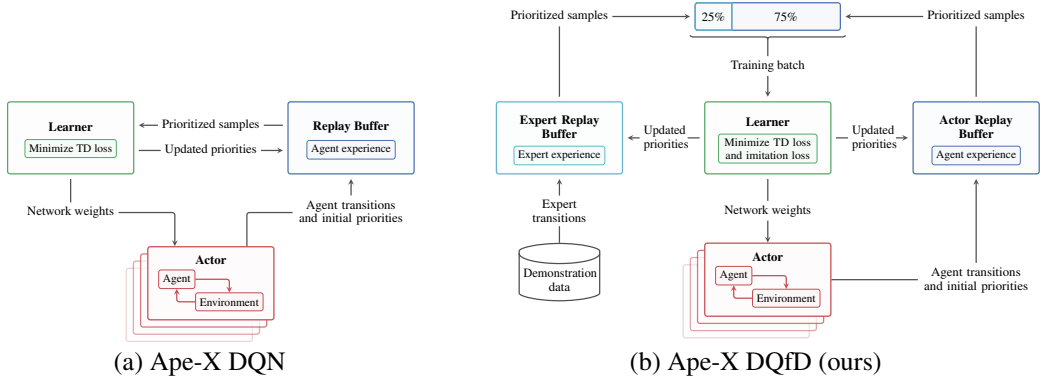
\begin{figure}[!t]
\resizebox{\linewidth}{!}{
\begin{tikzpicture}

\def \vdist {0.1}
\def \hdistmajor {2.0}

\node(learnertotal) {
	\begin{tikzpicture}
		\node (learner) {\textbf{Learner}};
		\node (learnersubtitle) [dmminor, color=dmgreen, fill=dmgreen!5, text=black, below=\vdist of learner, text width=2.5cm, align=center] {Minimize TD loss \\ and imitation loss};
		\begin{scope}[on background layer]
			\node (learnerbox) [dmmajor, color=dmgreen, fill=dmgreen!2, fit={(learner) (learnersubtitle)}] {};
		\end{scope}
	\end{tikzpicture}
};

\node(replaybuffertotal) [right=\hdistmajor of learnertotal] {
	\begin{tikzpicture}
		\node[text width=3cm, align=center] (replaybuffer) {\textbf{Actor Replay \\ Buffer}};
		\node (replaybuffersubtitle) [dmminor, color=dmblue, fill=dmblue!5, text=black, below=\vdist of replaybuffer] {Agent experience};
		\begin{scope}[on background layer]
			\node (replaybufferbox) [dmmajor, color=dmblue, fill=dmblue!2, fit={(replaybuffer) (replaybuffersubtitle)}] {};
		\end{scope}
	\end{tikzpicture}
};

\node(expreplaybuffertotal) [left=\hdistmajor of learnertotal] {
	\begin{tikzpicture}
		\node[text width=3cm, align=center] (expreplaybuffer) {\textbf{Expert Replay \\ Buffer}};
		\node (expreplaybuffersubtitle) [dmminor, color=dmcyan, fill=dmcyan!5, text=black, below=\vdist of expreplaybuffer] {Expert experience};
		\begin{scope}[on background layer]
			\node (expreplaybufferbox) [dmmajor, color=dmcyan, fill=dmcyan!2, fit={(expreplaybuffer) (expreplaybuffersubtitle)}] {};
		\end{scope}
	\end{tikzpicture}
};

\node[below=0.95 * \hdistmajor of learnertotal, anchor=north] (actor) {\textbf{Actor}};
\node(interaction) [below=-0.1cm of actor] {
	\begin{tikzpicture}
		\node (agent) [dmminor, line width=0.5pt, color=dmred, fill=dmred!5, text=black, ] {Agent};
		\node (env) [dmminor, line width=0.5pt, color=dmred, fill=dmred!5, text=black, below right=\vdist and 0 of agent, anchor=north west] {Environment};
		\path[dmpath, dmarrow, dmred] (agent.east) -| (env.north);
		\path[dmpath, dmarrow, dmred] (env.west) -| (agent.south);
	\end{tikzpicture}
};
\begin{scope}[on background layer]
	\node [dmmajor, color=dmred, fill=dmred!2, fit={(actor) (interaction)}, xshift=-0.6cm, yshift=-0.6cm, opacity=0.25] {};
	\node [dmmajor, color=dmred, fill=dmred!2, fit={(actor) (interaction)}, xshift=-0.4cm, yshift=-0.4cm, opacity=0.5] {};
	\node [dmmajor, color=dmred, fill=dmred!2, fit={(actor) (interaction)}, xshift=-0.2cm, yshift=-0.2cm, opacity=0.75] {};
	\node (actorbox) [dmmajor, color=dmred, fill=dmred!2, fit={(actor) (interaction)}] {};
\end{scope}

\node(batch) [above=1.3 * \hdistmajor of learner] {
	\begin{tikzpicture}
		\node[minimum width=1cm, dmlabel, minimum height=0.75cm] (p25) {25\%};
		\node[minimum width=3cm, dmlabel, minimum height=0.75cm, right=0 of p25] (p75) {75\%};
		
		\begin{scope}[on background layer]
			\fill[draw, color=dmcyan, fill=dmcyan!5] (p25.south east) {[rounded corners=2pt] -- (p25.south west) -- (p25.north west)} -- (p25.north east) -- (p25.south east);
			\fill[dmblue!5] (p75.south west) {[rounded corners=2pt] -- (p75.south east) -- (p75.north east)} -- (p75.north west) -- cycle;
			\draw[dmrect, dmblue] (p25.south west)  -- (p75.south east) -- (p75.north east) -- (p25.north west) -- cycle;
			\draw[dmpath, dmblue] (p25.south east) -- (p25.north east);
			\draw[dmpath, rounded corners=0, color=dmcyan] (p25.south east) {[rounded corners=2pt] -- (p25.south west) -- (p25.north west)} -- (p25.north east) -- cycle;
		\end{scope}
		
	\end{tikzpicture}
};

\draw [
	line width=1pt,
	color=dmlightgray,
    decoration={
        brace,
        mirror,
    },
    decorate
] (batch.south west) -- (batch.south east);
 
\node[draw, cylinder, shape border rotate=90, aspect=0.25,text width=2cm, align=center, dmlabel, dmpath, text=black] (db) at (expreplaybuffertotal.south |-, |-actorbox.west) {Demonstration \\ data};

\path[dmpath, dmarrow] (actorbox.east) -| (replaybuffertotal.south) node[midway, dmlabel, text=black, fill=white, text width=3cm, align=center]{Agent transitions \\and initial priorities};
\path[dmpath, dmarrow] (learnertotal.south) -- (actorbox.north) node[midway, dmlabel, text=black, fill=white]{Network weights};
\path[dmpath, dmarrow] (db.north) -- (expreplaybuffertotal.south)  node[midway, dmlabel, text=black, fill=white, text width=2cm, align=center]{Expert\\transitions};
\path[dmpath, dmarrow] (expreplaybuffertotal.north) |- (batch.west) node[midway, dmlabel, text=black, fill=white]{Prioritized samples};
\path[dmpath, dmarrow] (replaybuffertotal.north) |- (batch.east) node[midway, dmlabel, text=black, fill=white]{Prioritized samples};
\path[dmpath, dmarrow] (learnertotal.east) -- (replaybuffertotal.west) node[midway, dmlabel, text=black, fill=white, text width=1.1cm, align=center]{Updated\\priorities};
\path[dmpath, dmarrow] (learnertotal.west) -- (expreplaybuffertotal.east) node[midway, dmlabel, text=black, fill=white, text width=1.1cm, align=center]{Updated\\priorities};
\path[dmpath, dmarrow, shorten <=0.5mm] (batch.south) -- (learnertotal.north) node[midway, dmlabel, text=black, fill=white]{Training batch};

% % % % % % % % % % % % % % % % % % % % % % % % % % % % % %

\node(replaybuffertotal1) [left=\hdistmajor of expreplaybuffertotal] {
	\begin{tikzpicture}
		\node (replaybuffer1) {\textbf{Replay Buffer}};
		\node (replaybuffersubtitle1) [dmminor, color=dmblue, fill=dmblue!5, text=black, below=\vdist of replaybuffer1] {Agent experience};
		\begin{scope}[on background layer]
			\node (replaybufferbox1) [dmmajor, color=dmblue, fill=dmblue!2, fit={(replaybuffer1) (replaybuffersubtitle1)}] {};
		\end{scope}
	\end{tikzpicture}
};

\node(learnertotal1) [left=2 * \hdistmajor of replaybuffertotal1] {
	\begin{tikzpicture}
		\node (learner1) {\textbf{Learner}};
		\node (learnersubtitle1) [dmminor, color=dmgreen, fill=dmgreen!5, text=black, below=\vdist of learner1] {Minimize TD loss};
		\begin{scope}[on background layer]
			\node (learnerbox1) [dmmajor, color=dmgreen, fill=dmgreen!2, fit={(learner1) (learnersubtitle1)}] {};
		\end{scope}
	\end{tikzpicture}
};

\node (actor1) [below right=\hdistmajor and \hdistmajor of learnertotal1, anchor=north] {\textbf{Actor}};
\node(interaction1) [below=-0.1cm of actor1] {
	\begin{tikzpicture}
		\node (agent1) [dmminor, line width=0.5pt, color=dmred, fill=dmred!5, text=black, ] {Agent};
		\node (env1) [dmminor, line width=0.5pt, color=dmred, fill=dmred!5, text=black, below right=\vdist and 0 of agent1, anchor=north west] {Environment};
		\path[dmpath, dmarrow, dmred] (agent1.east) -| (env1.north);
		\path[dmpath, dmarrow, dmred] (env1.west) -| (agent1.south);
	\end{tikzpicture}
};
\begin{scope}[on background layer]
	\node [dmmajor, color=dmred, fill=dmred!2, fit={(actor1) (interaction1)}, xshift=-0.6cm, yshift=-0.6cm, opacity=0.25] {};
	\node [dmmajor, color=dmred, fill=dmred!2, fit={(actor1) (interaction1)}, xshift=-0.4cm, yshift=-0.4cm, opacity=0.5] {};
	\node [dmmajor, color=dmred, fill=dmred!2, fit={(actor1) (interaction1)}, xshift=-0.2cm, yshift=-0.2cm, opacity=0.75] {};
	\node (actorbox1) [dmmajor, color=dmred, fill=dmred!2, fit={(actor1) (interaction1)}] {};
\end{scope}

\node (middle1) at ($(actorbox1.north)!0.5!(replaybuffertotal1.south)$) {};

\node[dmlabel, text width=3cm, align=center] (actorreplaylabel1) at (replaybuffertotal1.south |-, |- middle1) {Agent transitions \\and initial priorities};
\node[dmlabel] (learneractorlabel1) at (learnertotal1.south |-, |- middle1) {Network weights};

\path[dmarrow, dmpath] ([yshift=-0.35cm]learnertotal1.east) -- ([yshift=-0.35cm]replaybuffertotal1.west) node[midway, fill=white, dmlabel, text=black]{Updated priorities};
\path[dmpath, dmarrow] ([yshift=0.35cm]replaybuffertotal1.west) -- ([yshift=0.35cm]learnertotal1.east) node[midway, fill=white, dmlabel, text=black]{Prioritized samples};

\path[dmpath] ([xshift=1cm]actorbox1.north) |- (actorreplaylabel1.west);
\path[dmpath, dmarrow] (actorreplaylabel1.north) -- (replaybuffertotal1.south);

\path[dmpath, dmarrow] (learneractorlabel1.east) -| ([xshift=-1cm]actorbox1.north);
\path[dmpath] (learneractorlabel1.north) -- (learnertotal1.south);

\end{tikzpicture}
}
\begin{tabular}[]{P{4.9cm}P{1.7cm}P{4.8cm}}
(a) Ape-X DQN &  & (b) Ape-X DQfD (ours)
\end{tabular}
\caption{The figure compares our architecture (b) to the one proposed by \citet{horgan2018distributed} (a).}
\label{fig:arch}

\end{figure}

In this section, we describe how we combine the transformed Bellman operator and the TC loss with the DQfD algorithm~\cite{hester2017learning} and distributed prioritized experience replay~\cite{horgan2018distributed}.
The resulting algorithm, which we call \emph{Ape-X DQfD} following \citet{horgan2018distributed}, is a distributed DQN algorithm with expert demonstrations that is robust to the reward distribution and can learn at discount factors an order of magnitude higher than what was possible before (\ie $\gamma = 0.999$ instead of $\gamma = 0.99$).
Our algorithm consists of three components: (1) replay buffers; (2) actor processes; and (3) a learner process. 
Fig.~\ref{fig:arch} shows how our architecture compares to the one used by \citet{horgan2018distributed}.

\textbf{Replay buffers.} Following \citet{hester2017learning}, we maintain two replay buffers: an \emph{actor replay buffer} and an \emph{expert replay buffer}.
Both buffers store 1-step and 10-step transitions and are prioritized~\cite{schaul2015prioritized}. 
The transitions in the actor replay buffer come from actor processes that interact with the MDP.
In order to limit the memory consumption of the actor replay buffer, we regularly remove transitions in a FIFO-manner. 
The expert replay buffer is filled once offline before training commences.

\textbf{Actor processes.} \citet{horgan2018distributed} showed that we can significantly improve the performance of DQN with prioritized replay buffers by having many actor processes.
We follow their approach and use $m = 128$ actor processes.
Each actor $i$ follows an $\varepsilon_i$-greedy policy based on the current estimate of the action-value function.
The noise levels $\varepsilon_i$ are chosen as $\varepsilon_i := 0.1^{\alpha_i + 3(1 - \alpha_i)}$ where $\alpha_i := \frac{i - 1}{m - 1}$. 
Notably, this exploration is closer to the one used by \citet{hester2017learning} and is much lower (\ie less random exploration) than the schedule used by \citet{horgan2018distributed}.

\textbf{Learner process.} The learner process samples experiences from the two replay buffers and minimizes a loss in order to approximate the optimal action-value function.
Following \citet{hester2017learning}, we combine the TD-loss $L_{\text{TD}}$ with a supervised imitation loss.
Let $t_1,...,t_N$ be transitions of the form $t_i = (x_i, a_i, r'_i, x'_i, e_i)$ with normalized priorities $p_1,...,p_N$ where $e_i$ is 1 if the transition is part of the best (\ie highest episode return) expert episode and 0 otherwise.
The imitation loss is a max-margin loss of the form
\begin{align}
L_{\text{IM}}(\theta; (t_i)_{i = 1}^N, (p_i)_{i = 1}^N, \theta^{(k - 1)}) := \sum_{i = 1}^{N} p_i e_i \left(\max_{a \in \mathcal{A}} [f_\theta(x_i, a) + \lambda \delta_{a \neq a_i}] - f_\theta(x_i, a_i) \right) \label{eq:maxmargin}
\end{align}
where $\lambda \in \mathbb{R}$ is the \emph{margin} and $\delta_{a \neq a_i}$ is 1 if $a \neq a_i$ and 0 otherwise.
Combining the imitation loss with the TD loss and the TC loss yields the total loss formulation
\begin{align*}
L(\theta; (t_i)_{i = 1}^N, (p_i)_{i = 1}^N, \theta^{(k - 1)}) := (L_{\text{TD}} + L_{\text{TC}} + L_{\text{IM}})(\theta; (t_i)_{i = 1}^N, (p_i)_{i = 1}^N, \theta^{(k - 1)}).
\end{align*}
Algo.~\ref{alg:learner}, provided in the appendix,
%Algo.~B.1, provided in the supplemental material, %Submission
shows the entire learner procedure.
Note that while we only apply the imitation loss $L_\text{IM}$ on the best expert trajectory, we still use all expert trajectories for the other two losses.

Our learning algorithm differs from the one used by \citet{hester2017learning} in three important ways. 
First, we do not have a pre-training phase where we minimize $L$ only using expert transitions.
We learn with a mix of actor and expert transitions from the beginning. 
Second, we maintain a fixed ratio of actor and expert transitions. 
For each SGD step, our training batch consists of 75\% agent transitions and 25\% expert transitions.
The ratio is constant throughout the entire learning process.
Finally, we only apply the imitation loss $L_\text{IM}$ to the best expert episode instead of all episodes.

\begin{table*}[!t]
\tiny
\begin{tabularx}{\linewidth}{r?C-C-C?C-C?C-C-C}
\toprule
Algorithm& \RotText{Rainbow} & \RotText{DQfD } & \RotText{Ape-X DQN} & \RotText{Ape-X DQfD} & \RotText{\parbox{1.3cm}{Ape-X DQfD \\ (deeper)}} & \RotText{Random} & \RotText{Avg. Human} & \RotText{\parbox{1.3cm}{Best Expert\\Trajectory}} \\
\midrule
Rainbow DQN  & -- & 31 / 42 & 9 / 42 & 10 / 42 & 7 / 42 & 41 / 42 & 32 / 42 & 24 / 42\\
DQfD  & 11 / 42 & -- & 7 / 42 & 11 / 42 & 2 / 42 & 40 / 42 & 25 / 42 & 13 / 42 \\
Ape-X DQN  & 34 / 42 & 35 / 42 & -- & 28 / 42 & \textbf{15 / 42} & 40 / 42 & 31 / 42 & 31 / 42 \\
\midrule
Ape-X DQfD  & 32 / 42 & 39 / 42 & 15 / 42 & -- & 9 / 42 & 40 / 42 & 39 / 42 & 32 / 42\\
Ape-X DQfD (deeper) & \textbf{36 / 42} & \textbf{40 / 42} & \textbf{28 / 42} & \textbf{33 / 42} & -- & \textbf{42 / 42} & \textbf{40 / 42} & \textbf{34 / 42} \\
\bottomrule
\end{tabularx}
\caption{The table shows on which fraction of the tested games one approach performs at least as well as the other.
The scores used for the comparison are using the no-op starts regime. 
As described in Sec.~\ref{sec:Baseline}, we compare the agents' scores to the scores obtained by an average human player and an expert player. 
Ape-X DQfD (deeper) out-performs the average human on 40 of 42 games.}
\label{tbl:games_comp}

\vspace{5mm}

\begin{tabularx}{\linewidth}{r?C-C-C-C?C-C-C-C}
\toprule
& \multicolumn{4}{c?}{No-op starts}& \multicolumn{4}{c}{Human starts} \\
&\multicolumn{2}{c-}{Mean} & \multicolumn{2}{c?}{Median} & \multicolumn{2}{c-}{Mean} & \multicolumn{2}{c}{Median} \\
Algorithm& 42 Games & 57 Games& 42 Games & 57 Games& 42 Games & 57 Games& 42 Games & 57 Games\\
\midrule
Rainbow DQN & 1022\% & 874\% & 231\% & 231\% & 897\% & 776\% & 159\% & 153\%\\ 
DQfD   & 364\% & -- & 113\% & -- & -- & -- & --& -- \\ 
Ape-X DQN   & 1770\% & 1695\% & 421\% & 434\% & 1651\% & 1591\% & 354\% & 358\%\\ 
\midrule
Ape-X DQfD   & 1536\% & -- & 339\% & -- & 1461\% & -- & 302\% & --\\ 
Ape-X DQfD (deeper)  & \textbf{2346\%} & -- & \textbf{702\%} & -- & \textbf{2028\%} & -- & \textbf{547\%} & --\\ 
\bottomrule
\end{tabularx}
\caption{The table shows the human-normalized performance of our algorithm and the baselines. 
For each game, we normalize the score as $\frac{\text{alg. score} - \text{random score}}{\text{avg. human score} - \text{random score}} \times 100$ and then aggregate over all games (mean or median).
Because we only have demonstrations on 42 out of the 57 games, we report the performances on 42 games and also 57 games for baselines not using demonstrations.}
\label{tbl:scores_comp}
\vspace{-3mm}
\end{table*}

\section{Experimental evaluation}
\label{sec:experiments}

We evaluate our approach on the same subset of 42 games from the Arcade Learning Environment (ALE)~\cite{bellemare2015arcade} used by \citet{hester2017learning}.
We report the performance using the \emph{no-op starts} and the \emph{human starts} test regimes~\cite{mnih2015human}. 
The full evaluation procedure is detailed in Sec.~\ref{sec:setup}.
%The full evaluation procedure is detailed in Sec.~C. %submission

\subsection{Benchmark results}
\label{sec:Baseline}
We compare our approach to Ape-X DQN~\cite{horgan2018distributed}, on which our actor-learner architecture is based, DQfD~\cite{hester2017learning}, which introduced the expert replay buffer and the imitation loss, and Rainbow DQN~\cite{hessel2017rainbow}, which combines all major DQN extensions from the literature into a single algorithm. 
Note that the scores reported in \cite{horgan2018distributed} were obtained by running 360 actors. 
Due to resource constraints, we limit the number of actors to 128 for all Ape-X DQfD experiments.
Besides comparing our performance to other RL agents, we are also interested in comparing our scores to a human player.
Because our demonstrations were gathered from an expert player, the expert scores are mostly better than the level of human performance reported in the literature~\cite{mnih2015human, wang2016dueling}. 
Hence, we treat the historical human scores as the performance of an average human and the scores of our expert as expert performance. 

We first analyse the performance of the standard dueling DQN architecture~\cite{wang2016dueling} that is also used by the baselines.
We report the scores as \emph{Ape-X DQfD} in Tables~\ref{tbl:games_comp}~and~\ref{tbl:scores_comp}.
We designed the algorithm to achieve higher consistency over a broad range of games and the scores shown in Table~\ref{tbl:games_comp} reflect that goal.
Whereas previous approaches outperformed an average human on at most 32 out of 42 games, Ape-X DQfD with the standard dueling architecture achieves a new state-of-the-art result of 39 out of 42 games.
This means we significantly improve the performance on the tails of the distribution of scores over the games.
When looking at this performance in the context of the median human-normalized scores reported in Table~\ref{tbl:scores_comp}, we see that we significantly increase the set of games where we learn good policies at the expense of achieving lower peak scores on some games.

One of the significant changes in our experimental setup is moving from a discount factor of $\gamma = 0.99$ to $\gamma = 0.999$.
\citet{jiang2015dependence} argue that this increases the complexity of the learning problem and, thus, requires a bigger hypothesis space.
Hence, in addition to the standard architecture, we also evaluated a slightly wider (\ie double the number of convolutional kernels) and deeper (one extra fully connected layer) network architecture (see 
Fig.~\ref{fig:networks}).
% Fig.~E.1). % Submission
With the deeper architecture, our algorithm outperforms an average human on 40 out of 42 games.
Furthermore, it is the first deep RL algorithm to learn non-trivial policies on all games including sparse reward games such as \textsc{Montezuma's Revenge}, \textsc{Private Eye}, and \textsc{Pitfall!}.
For example, we achieve 3997 points in \textsc{Pitfall!}, which is below the 6464 points of an average human but far above any baseline.
Finally, with a median human-normalized score of 702\% and exceeding every baseline on at least $\frac{2}{3}$ of the games, we demonstrate strong peak performance and consistency over the entire benchmark.

\subsection{Imitation vs. inspiration}

\begin{figure}[t]
\hspace{-4mm}
\includegraphics[width=1.04\linewidth]{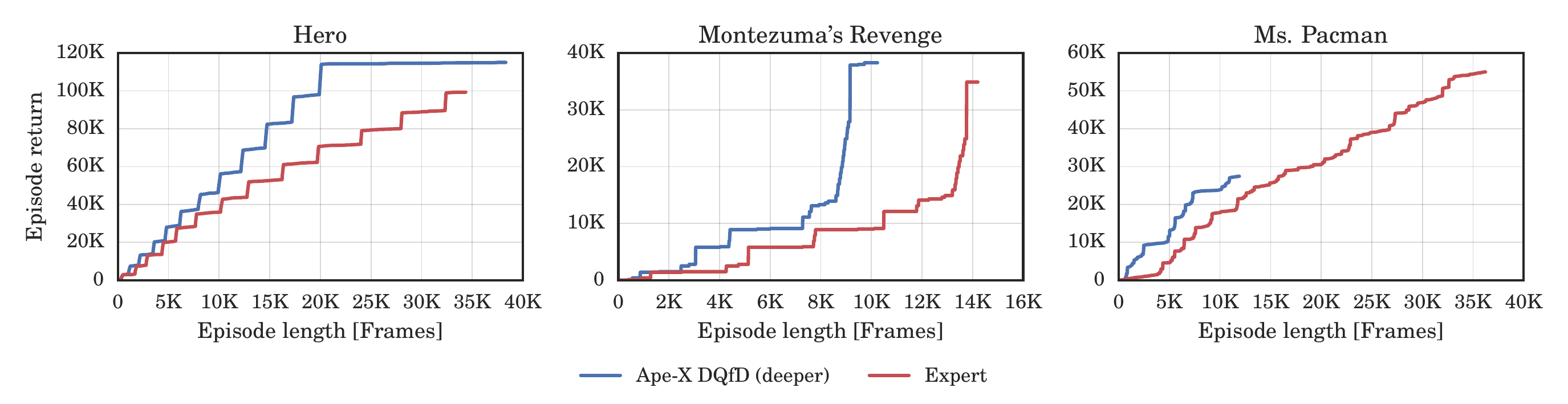}
\vspace{-6mm}
\caption{The figure shows the cumulative undiscounted episode return over time and compares the best expert episode to the best Ape-X DQfD episode on three games.
On \textsc{Hero}, the algorithm exceeds the expert's performance, on \textsc{Montezuma's Revenge}, it matches the expert's score but reaches it faster, and on \textsc{Ms. Pacman}, the expert is still superior.}
\label{fig:imitation}

\vspace{1mm}

\hspace{-4mm}
\includegraphics[width=1.04\linewidth]{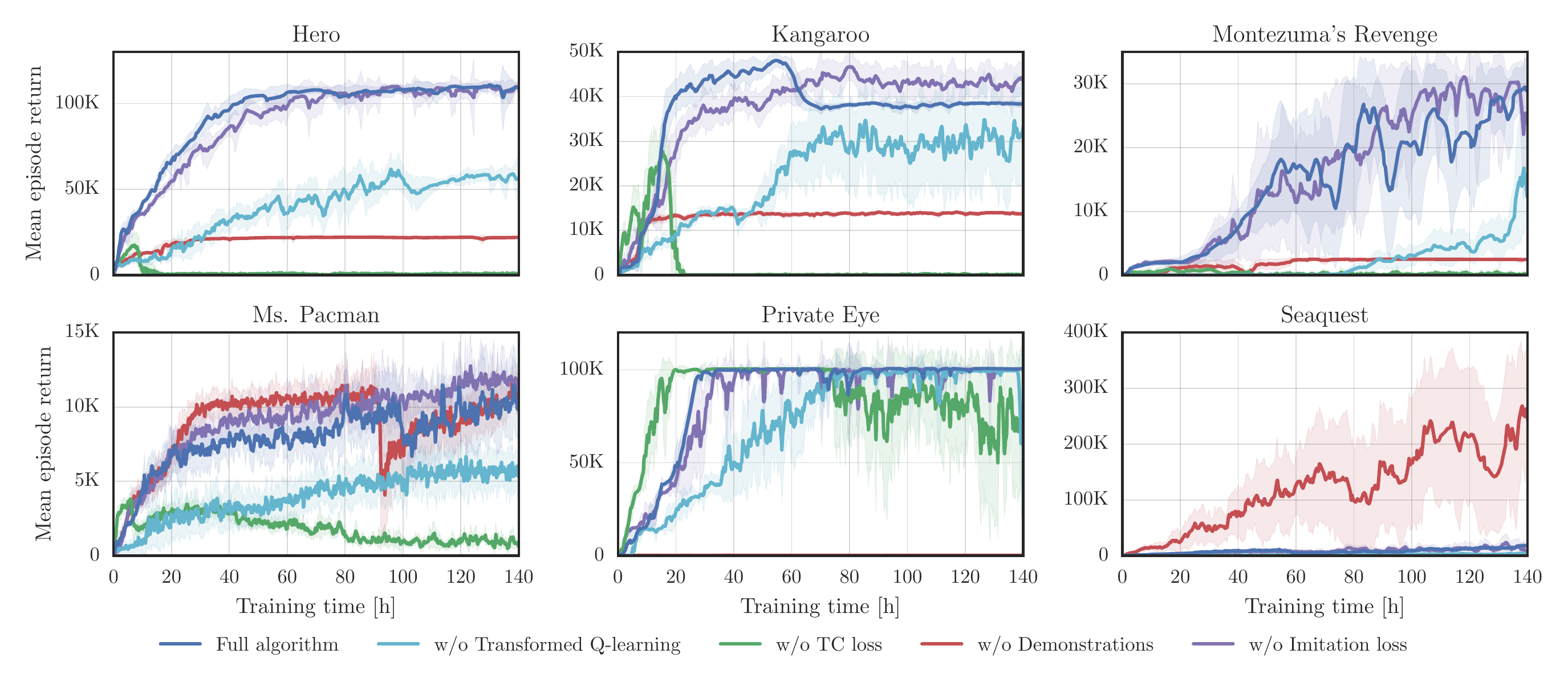}
\vspace{-6mm}
\caption{Results of our ablation study using the standard network architecture.
The experiment without expert data (\linereference{ablation_dqn}) was performed with the higher exploration schedule used in~\cite{horgan2018distributed}.}
\label{fig:ablation}
\vspace{-2mm}
\end{figure}

Although we use demonstration data, the goal of RLED algorithms is still to learn an optimal policy that maximizes the expected $\gamma$-discounted return.
While Table~\ref{tbl:games_comp} shows that we exceed the best expert episode on 34 games using the deeper architecture, it is hard to grasp the qualitative differences between the expert policies and our algorithm's policies. 
In order to qualitatively compare the agent and the expert, we provide videos
% in the supplemental material (see Sec.~F) % submission
on YouTube (see Sec.~\ref{sec:video})
and we plot the cumulative episode return of the best expert and agent episodes in Fig.~\ref{fig:imitation}.
We see that our algorithm (\linereference{ablation_dqfd}) finds more time-efficient policies than the expert (\linereference{ablation_dqn}) in all cases. 
This is a strong indicator that our algorithm does not do pure imitation but improves upon the demonstrated policies. %\thnote{Maybe we should add that on Hero, it is solving extra levels that the human did not solve.}

\subsection{Ablation study}

We evaluate the performance contributions of the three key ingredients of Ape-X DQfD (transformed Bellman operator, the TC-loss, and demonstration data) by performing an ablation study on a subset of 6 games.
We chose sparse-reward games (\textsc{Montezuma's Revenge}, \textsc{Private Eye}), dense-reward games (\textsc{Ms. Pacman}, \textsc{Seaquest}), and games where DQfD performs well (\textsc{Hero}, \textsc{Kangaroo}) (see Fig.~\ref{fig:ablation}).

\textbf{Transformed Bellman operator (\linereference{ablation_dqfd_transform}):}
When using the standard Bellman operator $\mathcal{T}$ instead of the transformed one, Ape-X DQfD is stable but the performance is significantly worse.

\textbf{TC loss (\linereference{ablation_dqfd_tc}):} In our setup, the TC loss is paramount to learning stably.
We see that without the TC loss the algorithm learns faster at the beginning of the training process.
However, at some point during training, the performance collapses and often the process dies with floating point exceptions.

\textbf{Expert demonstrations (\linereference{ablation_dqn} and \linereference{ablation_dqfd_expert_loss}):} 
Unsurprisingly, removing demonstrations entirely (\linereference{ablation_dqn}) severely degrades the algorithm's performance on sparse reward games. 
However, in games that an $\varepsilon$-greedy policy can efficiently explore, such as \textsc{Seaquest}, the performance is on par or worse.
Hence, the bias induced by the expert data is beneficial in some games and harmful in others.
Just removing the imitation loss $L_\text{IM}$ (\linereference{ablation_dqfd_expert_loss}) does not have a significant effect on the algorithm's performance.
This stands in contrast to the original DQfD algorithm and is most likely because we only apply the loss on a single expert trajectory.

\vspace{-1mm}
\subsection{Comparison to related work}
\vspace{-1mm}

\begin{figure}[t]
\hspace{-4mm}
\vspace{-2mm}
\includegraphics[width=1.04\linewidth]{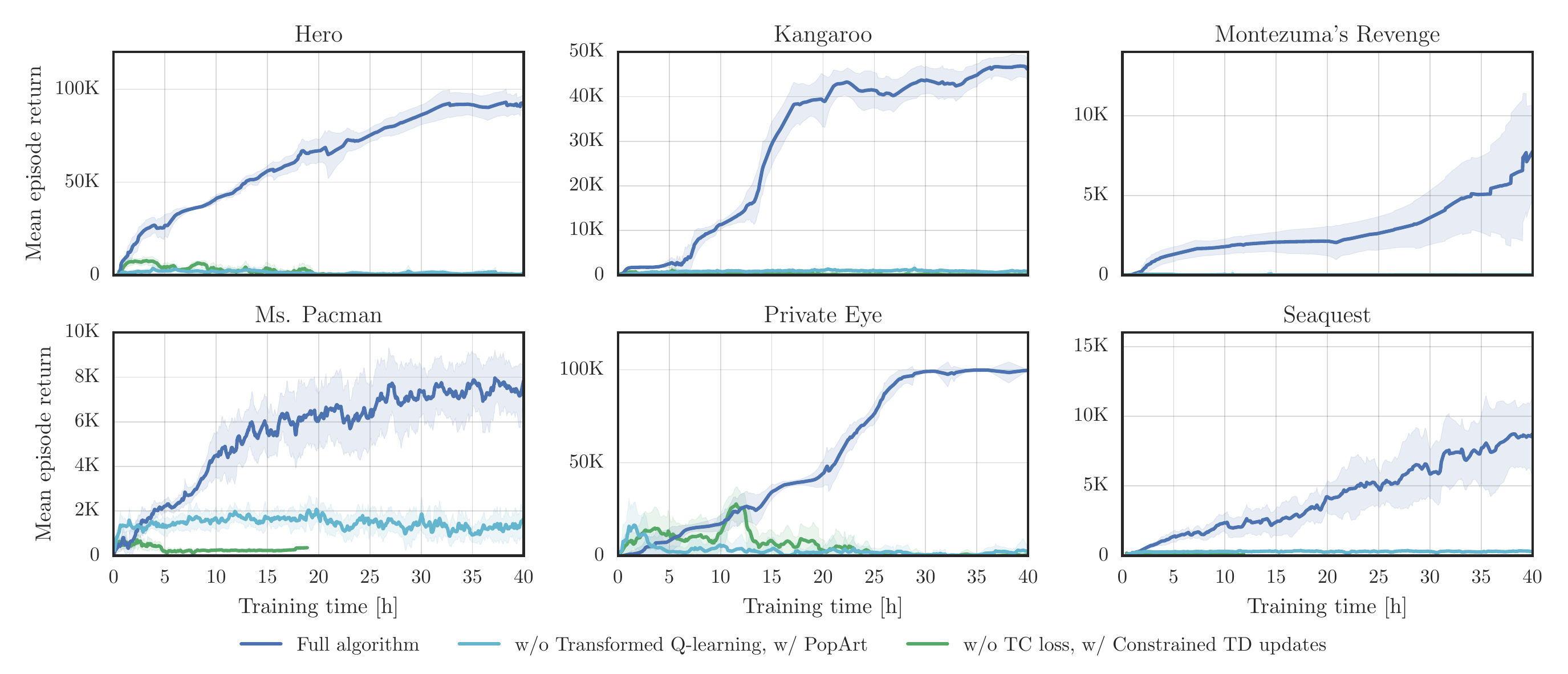}
\vspace{-5mm}
\caption{The figures show how our algorithm compares when we substitute the transformed Bellman operator to PopArt and when we substitute the TC loss to constrained TD updates.
Note that the scales differ from the ones in Fig.~\ref{fig:ablation} because the experiments only ran for 40 hours.}
\label{fig:related}
\vspace{-3mm}
\end{figure}
The problems of handling diverse reward distributions and network over-generalization in deep RL have been partially addressed in the literature (see Sec.~\ref{sec:RelatedWork}).
Specifically, \citet{vanhasselt2016popart} proposed PopArt and \citet{durugkar2017constrained} used constrained TD updates. 
We evaluate the performance of our algorithm when using alternative solutions and report the results in Fig.~\ref{fig:related}.

\textbf{PopArt (\linereference{ablation_dqfd_transform}):} 
We use the standard Bellman operator $\mathcal{T}$ in combination with PopArt, which adaptively normalizes the targets in (\ref{eq:dqn}) to have zero mean and unit variance.
While the modified algorithm manages to learn in some games, the overall performance is significantly worse than Ape-X DQfD.
One possible limiting factor that makes PopArt a bad choice for our framework is that training batches contain highly rewarding states from the very beginning of training.
SGD updates performed before the moving statistics have adequately adapted the moments of the target distribution might result in catastrophic changes to the network's weights. 
%Second, the TC loss ties our predictions at the next state to the target
%predictions. However, even if the actual targets do not change, the network might need to adapt 
%its predictions due to updated moving statistics. And finally, we employ a fixed margin $\lambda$
%in the imitation loss $L_\text{IM}$ that depends on the scale of the action-value function, which
%changes as the statistics are updated.

\textbf{Constrained TD updates (\linereference{ablation_dqfd_tc}):} 
We replaced the TC-loss with the constrained TD update approach \cite{durugkar2017constrained} that removes the target network and constrains the gradient to prevent an SGD update from changing the predictions at the next state.
We did not find the approach to work in our case.

\vspace{-1mm}
\section{Conclusion}
\vspace{-1mm}

In this paper, we presented a deep Reinforcement Learning (RL) algorithm that achieves human-level performance on a wide variety of MDPs on the Atari 2600 benchmark.
It does so by addressing three challenges:
handling diverse reward distributions, acting over longer time horizons, and efficiently exploring on sparse reward tasks. 
We introduce novel approaches for each of these challenges: 
a transformed Bellman operator, a temporal consistency loss, and a distributed RLED framework for learning from human demonstrations and task reward.
Our algorithm exceeds the performance of an average human on 40 out of 42 Atari 2600 games and it is the first deep RL algorithm to complete the first level of \textsc{Montezuma's Revenge}.

\bibliographystyle{plainnat}
\bibliography{references}

\begin{thebibliography}{22}
\providecommand{\natexlab}[1]{#1}
\providecommand{\url}[1]{\texttt{#1}}
\expandafter\ifx\csname urlstyle\endcsname\relax
  \providecommand{\doi}[1]{doi: #1}\else
  \providecommand{\doi}{doi: \begingroup \urlstyle{rm}\Url}\fi

\bibitem[Atkeson and Schaal(1997)]{atkeson1997robot}
Christopher Atkeson and Stefan Schaal.
\newblock Robot learning from demonstration.
\newblock In \emph{Proc. of ICML}, 1997.

\bibitem[Bellemare et~al.(2015)Bellemare, Naddaf, Veness, and
  Bowling]{bellemare2015arcade}
Marc Bellemare, Yavar Naddaf, Joel Veness, and Michael Bowling.
\newblock The {A}rcade {L}earning {E}nvironment: An evaluation platform for
  general agents.
\newblock In \emph{Proc. of IJCAI}, 2015.

\bibitem[Chemali and Lazaric(2015)]{chemali2015direct}
Jessica Chemali and Alessandro Lazaric.
\newblock Direct policy iteration with demonstrations.
\newblock In \emph{Proc. of IJCAI}, 2015.

\bibitem[Durugkar and Stone(2017)]{durugkar2017constrained}
Ishan Durugkar and Peter Stone.
\newblock {TD} learning with constrained gradients.
\newblock In \emph{Deep Reinforcement Learning Symposium, NIPS}, 2017.

\bibitem[Gruslys et~al.(2018)Gruslys, Dabney, Gheshlaghi~Azar, Piot, Bellemare,
  and Munos]{gruslys2018the}
Audrunas Gruslys, Will Dabney, Mohammad Gheshlaghi~Azar, Bilal Piot, Marc
  Bellemare, and Remi Munos.
\newblock The reactor: A fast and sample-efficient actor-critic agent for
  reinforcement learning.
\newblock In \emph{Proc. of ICLR}, 2018.

\bibitem[Hessel et~al.(2018)Hessel, Modayil, Van~Hasselt, Schaul, Ostrovski,
  Dabney, Horgan, Piot, Azar, and Silver]{hessel2017rainbow}
Matteo Hessel, Joseph Modayil, Hado Van~Hasselt, Tom Schaul, Georg Ostrovski,
  Will Dabney, Dan Horgan, Bilal Piot, Mohammad~G. Azar, and David Silver.
\newblock Rainbow: Combining improvements in deep reinforcement learning.
\newblock In \emph{Proc. of AAAI}, 2018.

\bibitem[Hester et~al.(2018)Hester, Vecerik, Pietquin, Lanctot, Schaul, Piot,
  Sendonaris, Dulac-Arnold, Osband, and Agapiou]{hester2017learning}
Todd Hester, Matej Vecerik, Olivier Pietquin, Marc Lanctot, Tom Schaul, Bilal
  Piot, Andrew Sendonaris, Gabriel Dulac-Arnold, Ian Osband, and John Agapiou.
\newblock Deep {Q}-learning from demonstrations.
\newblock \emph{Proc. of AAAI}, 2018.

\bibitem[Horgan et~al.(2018)Horgan, Quan, Budden, Barth-Maron, Hessel, van
  Hasselt, and Silver]{horgan2018distributed}
Dan Horgan, John Quan, David Budden, Gabriel Barth-Maron, Matteo Hessel, Hado
  van Hasselt, and David Silver.
\newblock Distributed prioritized experience replay.
\newblock In \emph{Proc. of ICLR}, 2018.

\bibitem[Huber(1964)]{huber1964robust}
Peter~J. Huber.
\newblock Robust estimation of a location parameter.
\newblock \emph{Ann. Math. Statist.}, 35\penalty0 (1):\penalty0 73--101, 03
  1964.

\bibitem[Jiang et~al.(2015)Jiang, Kulesza, Singh, and
  Lewis]{jiang2015dependence}
Nan Jiang, Alex Kulesza, Satinder Singh, and Richard Lewis.
\newblock The dependence of effective planning horizon on model accuracy.
\newblock In \emph{Proceedings of the 2015 International Conference on
  Autonomous Agents and Multiagent Systems}, pages 1181--1189. International
  Foundation for Autonomous Agents and Multiagent Systems, 2015.

\bibitem[Kim et~al.(2013)Kim, Farahmand, Pineau, and Precup]{kim2013learning}
Beomjoon Kim, Amir-massoud Farahmand, Joelle Pineau, and Doina Precup.
\newblock Learning from limited demonstrations.
\newblock In \emph{Proc. of NIPS}, 2013.

\bibitem[Lipton et~al.(2016)Lipton, Li, Gao, Li, Ahmed, and
  Deng]{lipton2017bbq}
Zachary Lipton, Xiujun Li, Jianfeng Gao, Lihong Li, Faisal Ahmed, and Li~Deng.
\newblock Bbq-networks: Efficient exploration in deep reinforcement learning
  for task-oriented dialogue systems.
\newblock In \emph{Proc. of AAAI}, 2016.

\bibitem[Mnih et~al.(2015)Mnih, Kavukcuoglu, Silver, Rusu, Veness, Bellemare,
  Graves, Riedmiller, Fidjeland, Ostrovski, et~al.]{mnih2015human}
Volodymyr Mnih, Koray Kavukcuoglu, David Silver, Andrei~A Rusu, Joel Veness,
  Marc~G Bellemare, Alex Graves, Martin Riedmiller, Andreas~K Fidjeland, Georg
  Ostrovski, et~al.
\newblock Human-level control through deep reinforcement learning.
\newblock \emph{Nature}, 518\penalty0 (7540):\penalty0 529--533, 2015.

\bibitem[Nair et~al.(2017)Nair, McGrew, Andrychowicz, Zaremba, and
  Abbeel]{nair2017overcoming}
Ashvin Nair, Bob McGrew, Marcin Andrychowicz, Wojciech Zaremba, and Pieter
  Abbeel.
\newblock Overcoming exploration in reinforcement learning with demonstrations.
\newblock \emph{arXiv preprint arXiv:1709.10089}, 2017.

\bibitem[Piot et~al.(2014)Piot, Geist, and Pietquin]{piot2014boosted}
Bilal Piot, Matthieu Geist, and Olivier Pietquin.
\newblock Boosted bellman residual minimization handling expert demonstrations.
\newblock In \emph{Proc. of ECML/PKDD}, 2014.

\bibitem[Puterman(1994)]{puterman1994markov}
Marc~L. Puterman.
\newblock \emph{{Markov decision processes: discrete stochastic dynamic
  programming}}.
\newblock John Wiley \& Sons, 1994.

\bibitem[Schaal(1997)]{schaal1997learning}
Stefan Schaal.
\newblock Learning from demonstration.
\newblock In \emph{Proc. of NIPS}, 1997.

\bibitem[Schaul et~al.(2015)Schaul, Quan, Antonoglou, and
  Silver]{schaul2015prioritized}
Tom Schaul, John Quan, Ioannis Antonoglou, and David Silver.
\newblock Prioritized experience replay.
\newblock In \emph{Proc. of ICLR}, 2015.

\bibitem[van Hasselt et~al.(2016{\natexlab{a}})van Hasselt, Guez, Hessel, Mnih,
  and Silver]{vanhasselt2016popart}
Hado van Hasselt, Arthur Guez, Matteo Hessel, Volodymyr Mnih, and David Silver.
\newblock Learning values across many orders of magnitude.
\newblock In \emph{Proc. of NIPS}, 2016{\natexlab{a}}.

\bibitem[van Hasselt et~al.(2016{\natexlab{b}})van Hasselt, Guez, and
  Silver]{van2016deep}
Hado van Hasselt, Arthur Guez, and David Silver.
\newblock Deep reinforcement learning with double {Q}-learning.
\newblock In \emph{Proc. of AAAI}, 2016{\natexlab{b}}.

\bibitem[Ve{\v{c}}er{\'\i}k et~al.(2017)Ve{\v{c}}er{\'\i}k, Hester, Scholz,
  Wang, Pietquin, Piot, Heess, Roth{\"o}rl, Lampe, and
  Riedmiller]{vevcerik2017leveraging}
Matej Ve{\v{c}}er{\'\i}k, Todd Hester, Jonathan Scholz, Fumin Wang, Olivier
  Pietquin, Bilal Piot, Nicolas Heess, Thomas Roth{\"o}rl, Thomas Lampe, and
  Martin Riedmiller.
\newblock Leveraging demonstrations for deep reinforcement learning on robotics
  problems with sparse rewards.
\newblock \emph{arXiv preprint arXiv:1707.08817}, 2017.

\bibitem[Wang et~al.(2016)Wang, Schaul, Hessel, Hasselt, Lanctot, and
  de~Freitas]{wang2016dueling}
Ziyu Wang, Tom Schaul, Matteo Hessel, Hado~van Hasselt, Marc Lanctot, and Nando
  de~Freitas.
\newblock Dueling network architectures for deep reinforcement learning.
\newblock In \emph{Proc. of ICML}, 2016.

\end{thebibliography}

\newpage

% Remove for camera ready.

\appendix

\section{Transformed Bellman Operator in Stochastic MDPs}
\label{sec:stochastic}

The following proposition shows that transformed Bellman operator is still a contraction  for  small $\gamma$ if we assume a stochastic MDP and a more generic choice of $h$. 
However, the fixed point might not be $h \circ Q^*$.

\begin{prop}
\label{prop:stochastic}
Let $h$ be strictly monotonically increasing, Lipschitz continuous with Lipschitz constant $L_h$, and have a Lipschitz continuous inverse $h^{-1}$ with  Lipschitz constant $L_{h^{-1}}$. For $\gamma < \frac{1}{L_hL_{h^{-1}}}$, $\mathcal{T}_h$ is a contraction.
\end{prop}
\begin{proof}
Let $Q, U : \mathcal{X} \times \mathcal{A} \to \mathbb{R}$ be arbitrary. It 
holds
\begin{align*}
\|\mathcal{T}_hQ - \mathcal{T}_hU\|_\infty =& \max_{x, a \in \mathcal{X} \times 
\mathcal{A}} \left| \mathbb{E}_{x' \sim P(\cdot | x, a)}\left[h\left( R(x, a) + 
\gamma \max_{a' \in \mathcal{A}} h^{-1}(Q(x', a')) \right) \right.\right.\\
&\left. \left. - h\left( R(x, a) 
+ \gamma \max_{a' \in \mathcal{A}} h^{-1}(Q(x', a')) - \right)\right]\right| \\
\stackrel{(1)}{\leq} & L_h \gamma \max_{x, a \in \mathcal{X} \times 
\mathcal{A}} 
\mathbb{E}_{x' \sim P(\cdot | x, a)}\left[ \left| \max_{a' \in \mathcal{A}} 
h^{-1}(Q(x', a')) - \max_{a' \in \mathcal{A}} h^{-1}(U(x', a'))\right|\right]\\
\leq & L_h \gamma \max_{x, a \in \mathcal{X} \times \mathcal{A}} 
\mathbb{E}_{x' \sim P(\cdot | x, a)}\left[ \max_{a' \in \mathcal{A}}\left|  
h^{-1}(Q(x', a')) - h^{-1}(U(x', a'))\right|\right] \\
\stackrel{(2)}{\leq} & L_h L_{h^{-1}} \gamma \max_{x, a \in \mathcal{X} \times 
\mathcal{A}} 
\mathbb{E}_{x' \sim P(\cdot | x, a)}\left[ \max_{a' \in \mathcal{A}}\left|  
Q(x', a') - U(x', a')\right|\right] \\
\leq & \underbrace{L_h L_{h^{-1}} \gamma}_{<1} \|Q - U\|_\infty < \|Q - 
U\|_\infty
\end{align*}
where we used Jensen's inequality in (1) and the Lipschitz properties of $h$ 
and $h^{-1}$ in (1) and (2).
\end{proof}
For our algorithm, we use $h : \mathbb{R} \to \mathbb{R}, x \mapsto 
\text{sign}(x)(\sqrt{|x| + 1} - 1) + \varepsilon x$ with $\varepsilon = 
10^{-2}$. While Proposition~\ref{prop:h} shows that the transformed operator is a contraction, the discount factor 
$\gamma$ we use in practice is higher than $\frac{1}{L_hL_{h^{-1}}}$. We leave a deeper 
investigation of the contraction properties of $\mathcal{T}_h$ in stochastic MDPs
for future work.
\begin{prop}
\label{prop:h}
Let $\varepsilon > 0$ and $h : \mathbb{R} \to \mathbb{R}, x \mapsto 
\text{sign}(x)(\sqrt{|x| + 1} - 1) + \varepsilon x$. It holds
\begin{enumerate}[(i)]
\item $h$ is strictly monotonically increasing.
\item $h$ is Lipschitz continuous with Lipschitz constant $L_h = \frac{1}{2} + 
\varepsilon$.
\item $h$ is invertible with $h^{-1} : \mathbb{R} \to \mathbb{R}, x \mapsto 
\text{sign}(x) \left( \left( \frac{\sqrt{1 + 4\varepsilon (|x| + 1 + 
\varepsilon)} - 1}{2 \varepsilon} \right)^2 -1 \right)$.
\item $h^{-1}$ is strictly monotonically increasing.
\item $h^{-1}$ is Lipschitz continuous with Lipschitz constant $L_{h^{-1}} = 
\frac{1}{\varepsilon}$.
\end{enumerate}
\end{prop}
We use the following Lemmas in order to prove Proposition~\ref{prop:h}.
\begin{lemma}
\label{lemma:h}
$h : \mathbb{R} \to \mathbb{R}, x \mapsto 
\text{sign}(x)(\sqrt{|x| + 1} - 1) + \varepsilon x$ is differentiable
everywhere with derivative $\frac{d}{dx}h(x) = \frac{1}{2\sqrt{|x| + 1}} + \varepsilon$
for all $x \in \mathbb{R}$.
\end{lemma}
\begin{proof}[Proof of Lemma~\ref{lemma:h}]
For $x > 0$, $h$ is differentiable as a composition of differentiable functions with
$\frac{d}{dx}h(x) = \frac{1}{2\sqrt{x + 1}} + \varepsilon$. Analogously, $h$ is 
differentiable for $x < 0$ with $\frac{d}{dx}h(x) = \frac{1}{2\sqrt{-x + 1}}+ \varepsilon$.
For $x = 0$, we find
\begin{align*}
    \lim_{z \rightarrow 0^+} \frac{h(x + z) - \overbrace{h(x)}^{=0}}{z} = \lim_{z \rightarrow 0^+} \frac{\sqrt{z + 1} - 1 + \varepsilon z}{z} \stackrel{\text{l'Hospital}}{=} \lim_{z \rightarrow 0^+} \overbrace{\frac{1}{2\sqrt{z + 1}}}^{\rightarrow \frac{1}{2}} + \varepsilon = \frac{1}{2} + \varepsilon
\end{align*}
and similarly $\lim_{z \rightarrow 0^-} \frac{h(x + z) - h(x)}{z} = \frac{1}{2} + \varepsilon$.
Hence, $\frac{d}{dx}h(x) = \frac{1}{2\sqrt{|x| + 1}} + \varepsilon$ for all $x \in \mathbb{R}$.
\end{proof}
\begin{lemma}
\label{lemma:h_inv}
$h^{-1} : \mathbb{R} \to \mathbb{R}, x \mapsto 
\text{sign}(x) \left( \left( \frac{\sqrt{1 + 4\varepsilon (|x| + 1 + 
\varepsilon)} - 1}{2 \varepsilon} \right)^2 -1 \right)$ is differentiable everywhere
with derivative $\frac{d}{dx}h^{-1}(x) = \frac{1}{\varepsilon}\left(1 - \frac{1}{\sqrt{1 + 4\varepsilon(|x| + 1 + \varepsilon)}} \right)$ for all $x \in \mathbb{R}$.
\end{lemma}
\begin{proof}[Proof of Lemma~\ref{lemma:h_inv}]
For $x \neq 0$, $h^{-1}$ is differentiable as a composition of differentiable functions.
For $x = 0$, it holds
\begin{align*}
    \lim_{z \rightarrow 0^+} \frac{h^{-1}(x + z) - \overbrace{h^{-1}(x)}^{=0}}{z} =& \lim_{z \rightarrow 0^+} \frac{1}{z} \frac{(\sqrt{1 + 4\varepsilon(z + 1 + \varepsilon)} - 1)^2 - 1}{4\varepsilon^2} \\
    \stackrel{\text{l'Hospital}}{=}& \lim_{z \rightarrow 0^+} \frac{1}{4\varepsilon^2} \frac{2(\sqrt{1 + 4\varepsilon(z + 1 + \varepsilon)} - 1)}{2\sqrt{1 + 4\varepsilon(z + 1 + \varepsilon)}} 4\varepsilon \\
    =& \lim_{z \rightarrow 0^+} \frac{1}{\varepsilon} \left( 1 - \frac{1}{\sqrt{1 + 4\varepsilon(z + 1 + \varepsilon)}}\right) = \frac{1}{\frac{1}{2} + \varepsilon}.
\end{align*}
and similarly $\lim_{z \rightarrow 0^-} \frac{h^{-1}(x + z) - h^{-1}(x)}{z} = \frac{1}{\frac{1}{2} + \varepsilon}$, which concludes the proof.
\end{proof}
\begin{proof}[Proof of Proposition~\ref{prop:h}]
We prove all statements individually.
\begin{enumerate}[(i)]
    \item $\frac{d}{dx}h(x) = \frac{1}{2\sqrt{|x| + 1}} + \varepsilon x > 0$ for all $x \in
    \mathbb{R}$, which implies the proposition.
    \item Let $x, y \in \mathbb{R}$ with $x < y$, using the mean value theorem, we find
    \begin{align*}
        |h(x) - h(y)| \leq \sup_{\xi \in (x, y)} \left| \frac{d}{dx}h(\xi) \right| |x - y| \leq \sup_{\xi \in \mathbb{R}}\left| \frac{d}{dx}h(\xi) \right| |x - y| = \underbrace{\left( \frac{1}{2} + \varepsilon \right)}_{= L_h} |x - y|.
    \end{align*}
    \item (i) Implies that $h$ is invertible and simple substitution shows $h \circ h^{-1} = h^{-1} \circ h = \text{id}$.
    \item $\frac{d}{dx}h^{-1}(x) = \frac{1}{\varepsilon}\left(1 - \frac{1}{\sqrt{1 + 4\varepsilon(|x| + 1 + \varepsilon)}} \right) > 0$ for all $x \in \mathbb{R}$, which implies the proposition.
    \item Let $x, y \in \mathbb{R}$ with $x < y$, using the mean value theorem, we find
    \begin{align*}
        |h^{-1}(x) - h^{-1}(y)| \leq \sup_{\xi \in (x, y)} \left| \frac{d}{dx}h^{-1}(\xi) \right| |x - y| \leq \sup_{\xi \in \mathbb{R}}\left| \frac{d}{dx}h^{-1}(\xi) \right| |x - y| = \underbrace{\frac{1}{\varepsilon}}_{= L_{h^{-1}}} |x - y|.
    \end{align*}
    
\end{enumerate}
\end{proof}

\newpage
\section{Learner algorithm}
\begin{algorithm}[!h]
\caption{The algorithm used by the learner to estimate the action-value 
function.}
\label{alg:learner}
\begin{small}
\begin{algorithmic}
\State $\theta^{(0)} \leftarrow$ Random sample
\For{$k = 1,2,...$}
	\State $\theta^{(k)} \leftarrow \theta^{(k - 1)}$
	\For{$j = 1,...,T_\text{update}$}
		\Comment{$T_\text{update}$ is the target network update period}
		\State $(t_i, p_i)_{i = 1}^N \leftarrow 
		$\Call{SamplePrioritized}{N}
		\Comment{Sample 75\% agent and 25\% expert transitions}		
		\State $\theta^{(k)} \leftarrow $ 
		\Call{AdamStep}{$L(\theta^{(k)}; 
		(t_i)_{i = 1}^N, (p_i)_{i = 1}^N, \theta^{(k - 1)})$}
		\Comment{Update the parameters using Adam}
		\State $(p_i)_{i = 1}^N \leftarrow \left(\left| 
		L_\text{TD}(\theta^{(k)}; 
				t_i, 1, \theta^{(k - 1)}) \right| \right)_{i = 1}^N$
		\Comment{Compute the updated priorities based on the TD error}
		\State \Call{UpdatePriorities}{$(t_1, p_1),...,(t_N, p_N)$}
		\Comment{Send the updated priorities to the replay buffers}
	\EndFor
\EndFor
\end{algorithmic}
\end{small}
\end{algorithm}
\section{Experimental setup}
\label{sec:setup}
We evaluate our algorithm on Arcade Learning Environment (ALE) by \citet{bellemare2015arcade}.
While we follow many of the practices commonly applied when training on the 
ALE, our experimental setup differs in a few key aspects from the 
defaults~\cite{mnih2015human,hessel2017rainbow}.

\paragraph{End episode on life loss.}
Most authors who train agents on the ALE choose to end a training episode when 
the agent loses a life. This naturally makes the agent
risk averse as an action that leads to the termination of an episode has a 
value of 0. However, because our expert player was allowed to continue playing 
an episode after losing a life, we follow~\citet{hester2017learning} and only 
terminate a training episode either when the game is over or when the agent has 
performed 50,000 steps which is the episode length used by  
\citet{horgan2018distributed}.

\paragraph{Reward Preprocessing. } As explained in Sec.~3.2,
% arXiv: ~\ref{sec:transform}, 
we do not clip the rewards to the interval $[-1, 1]$. Instead, we use the raw 
and unprocessed rewards provided by each game.

\paragraph{Discount factor. } The majority of approaches use a discount factor of
$\gamma = 0.99$. Empirically, this used to be the highest discount factor that 
allows stable learning on all games. However, the TC loss allows us to use a 
much higher discount factor of $\gamma = 0.999$ giving the algorithm an 
effective planning horizon of 1000 instead of 100 steps.

\paragraph{Expert data.} Instead of relying purely on an $\varepsilon$-greedy
exploration strategy, our algorithm uses expert demonstrations. By using these
demonstrations in the TD loss $L_\text{TD}$, the algorithm gets the experience
rewarding transitions without having to discover them itself.

\newpage
\FloatBarrier

\section{Full experimental results}

\begin{figure}[!h]
    \centering
    \vspace{-4mm}
    \includegraphics[width=0.97\linewidth]{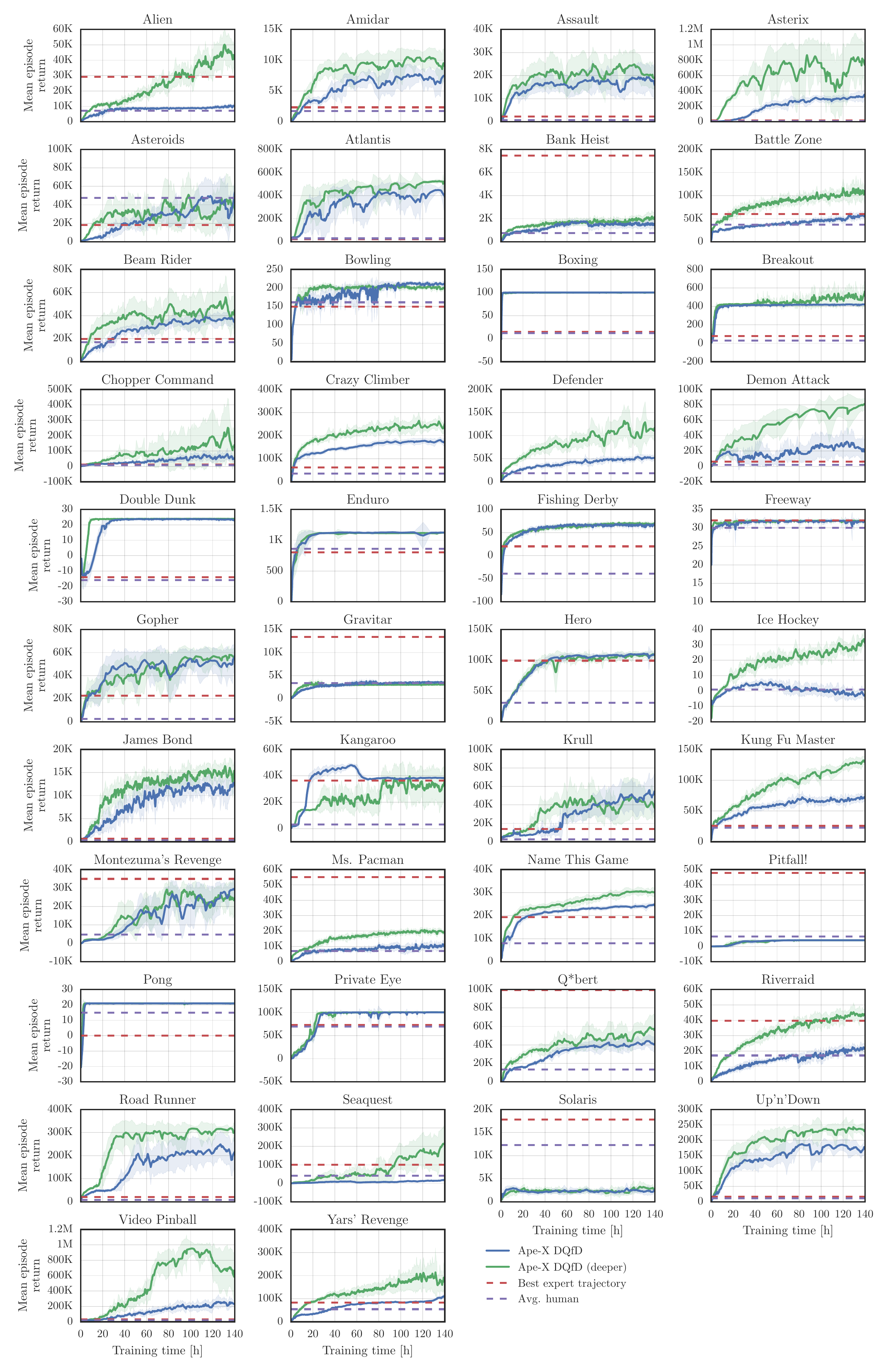}
    \vspace{-3mm}
    \caption{Training curves on all 42 games. We report the performance using the standard
    network architecture~\cite{wang2016dueling} and the slightly deeper version 
    (see Fig.~\ref{fig:networks}).}
    \label{fig:all_games}
\end{figure}

\newpage

\begin{table*}[!h]
\tiny
\begin{tabularx}{\linewidth}{r?R-R-R?R-R?R-R-R}
\toprule
Game& Rainbow & DQfD & Ape-X DQN & Ape-X DQfD & Ape-X DQfD (deeper) & Random & Avg. Human & Expert \\
\midrule
Alien &9491.7 & 4737.5 & 40804.9 & 11313.6 & \textbf{ 50113.6 } & 128.3 & 7128.0 & 29160.0\\
Amidar &5131.2 & 2325.0 & 8659.2 & 8463.8 & \textbf{ 12291.7 } & 11.8 & 1720.0 & 2341.0\\
Assault &14198.5 & 1755.7 & 24559.4 & 22855.0 & \textbf{ 35046.9 } & 166.9 & 742.0 & 2274.0\\
Asterix &\textbf{ 428200.3 } & 5493.6 & 313305.0 & 399888.0 & 418433.5 & 164.5 & 8503.0 & 18100.0\\
Asteroids &2712.8 & 3796.4 & \textbf{ 155495.1 } & 116846.4 & 112573.6 & 871.3 & 47389.0 & 18100.0\\
Atlantis &826659.5 & 920213.9 & 944497.5 & 911025.0 & \textbf{ 1057521.5 } & 13463.0 & 29028.0 & 22400.0\\
Bank  Heist &1358.0 & 1280.2 & 1716.4 & 2061.9 & 2578.9 & 21.7 & 753.0 & \textbf{ 7465.0 }\\
Battle Zone &62010.0 & 41708.2 & 98895.0 & 60540.0 & \textbf{ 128925.0 } & 3560.0 & 37188.0 & 60000.0\\
Beam Rider &16850.2 & 5173.3 & 63305.2 & 47129.4 & \textbf{ 87257.4 } & 254.6 & 16926.0 & 19844.0\\
Bowling &30.0 & 97.0 & 17.6 & \textbf{ 216.3 } & 210.9 & 35.2 & 161.0 & 149.0\\
Boxing &99.6 & 99.1 & \textbf{ 100.0 } & \textbf{ 100.0 } & 98.5 & -0.1 & 12.0 & 15.0\\
Breakout &417.5 & 308.1 & \textbf{ 800.9 } & 419.7 & 641.9 & 1.6 & 30.0 & 79.0\\
Chopper Command &16654.0 & 6993.1 & 721851.0 & 96653.0 & \textbf{ 840023.5 } & 644.0 & 7388.0 & 11300.0\\
Crazy Climber &168788.5 & 151909.5 & \textbf{ 320426.0 } & 176598.5 & 247651.0 & 9337.0 & 35829.0 & 61600.0\\
Defender &55105.0 & 27951.5 & \textbf{ 411943.5 } & 51442.0 & 218006.3 & 1965.5 & 18689.0 & 18700.0\\
Demon Attack &111185.2 & 3848.8 & 133086.4 & 100200.9 & \textbf{ 141444.6 } & 208.3 & 1971.0 & 6190.0\\
Double Dunk &-0.3 & -20.4 & \textbf{ 23.5 } & 23.0 & 23.2 & -16.0 & -16.0 & -14.0\\
Enduro &2125.9 & 1929.8 & \textbf{ 2177.4 } & 1663.1 & 1910.1 & 81.8 & 860.0 & 803.0\\
Fishing Derby &31.3 & 38.4 & 44.4 & 66.1 & \textbf{ 68.0 } & -77.1 & -39.0 & 20.0\\
Freeway &\textbf{ 34.0 } & 31.4 & 33.7 & 32.0 & 31.7 & 0.1 & 30.0 & 32.0\\
Gopher &70354.6 & 7810.3 & \textbf{ 120500.9 } & 114702.6 & 114168.9 & 250.0 & 2412.0 & 22520.0\\
Gravitar &1419.3 & 1685.1 & 1598.5 & 4214.3 & 3920.5 & 245.5 & 3351.0 & \textbf{ 13400.0 }\\
Hero &55887.4 & 105929.4 & 31655.9 & 112042.4 & \textbf{ 114248.2 } & 1580.3 & 30826.0 & 99320.0\\
Ice Hockey &1.1 & -9.6 & \textbf{ 33.0 } & 3.4 & 32.9 & -9.7 & 1.0 & 1.0\\
James Bond &19809.0 & 2095.0 & \textbf{ 21322.5 } & 12889.0 & 16956.3 & 33.5 & 303.0 & 650.0\\
Kangaroo &14637.5 & 14681.5 & 1416.0 & 47676.5 & \textbf{ 48599.0 } & 100.0 & 3035.0 & 36300.0\\
Krull &8741.5 & 9825.3 & 11741.4 & 104160.3 & \textbf{ 140670.6 } & 1151.9 & 2666.0 & 13730.0\\
Kung Fu Master &52181.0 & 29132.0 & 97829.5 & 67957.5 & \textbf{ 137804.5 } & 304.0 & 22736.0 & 25920.0\\
Montezuma's Revenge &384.0 & 4638.4 & 2500.0 & 29384.0 & 27926.5 & 25.0 & 4753.0 & \textbf{ 34900.0 }\\
Ms. Pacman &5380.4 & 4695.7 & 11255.2 & 12857.1 & 20872.7 & 197.8 & 6952.0 & \textbf{ 55021.0 }\\
Name This Game &13136.0 & 5188.3 & 25783.3 & 24465.8 & \textbf{ 31569.4 } & 1747.8 & 8049.0 & 19380.0\\
Pitfall! &0.0 & 57.3 & -0.6 & 3996.7 & 3997.5 & -348.8 & 6464.0 & \textbf{ 47821.0 }\\
Pong &20.9 & 10.7 & 20.9 & \textbf{ 21.0 } & 20.9 & -18.0 & 15.0 & 0.0\\
Private Eye &4234.0 & 42457.2 & 49.8 & \textbf{ 100747.4 } & 100724.9 & 662.8 & 69571.0 & 72800.0\\
Q*bert &33817.5 & 21792.7 & \textbf{ 302391.3 } & 71224.4 & 91603.5 & 183.0 & 13455.0 & 99450.0\\
Riverraid &22920.8 & 18735.4 & \textbf{ 63864.4 } & 24147.7 & 47609.9 & 588.3 & 17118.0 & 39710.0\\
Road Runner &62041.0 & 50199.6 & 222234.5 & 507213.0 & \textbf{ 578806.5 } & 200.0 & 7845.0 & 20200.0\\
Seaquest &15898.9 & 12361.6 & \textbf{ 392952.3 } & 13603.8 & 318418.0 & 215.5 & 42055.0 & 101120.0\\
Solaris &3560.3 & 2616.8 & 2892.9 & 2529.8 & 3428.9 & 2047.2 & 12327.0 & \textbf{ 17840.0 }\\
Up'n'Down &125754.6 & 82555.0 & 401884.3 & 324505.2 & \textbf{ 469548.3 } & 707.2 & 11693.0 & 16080.0\\
Video Pinball &533936.5 & 19123.1 & 565163.2 & 243320.1 & \textbf{ 922518.0 } & 20452.0 & 17668.0 & 32420.0\\
Yars' Revenge &102557.0 & 61575.7 & 148594.8 & 109980.9 & \textbf{ 498947.1 } & 1476.9 & 54577.0 & 83523.0\\
\bottomrule
\end{tabularx}
\caption{Scores obtained by evaluating the best checkpoint for 200 episodes using the
no-op starts regime.}
\label{tbl:all_games_random}

\end{table*}

\begin{figure}[!h]
    \centering
    \includegraphics[width=\linewidth]{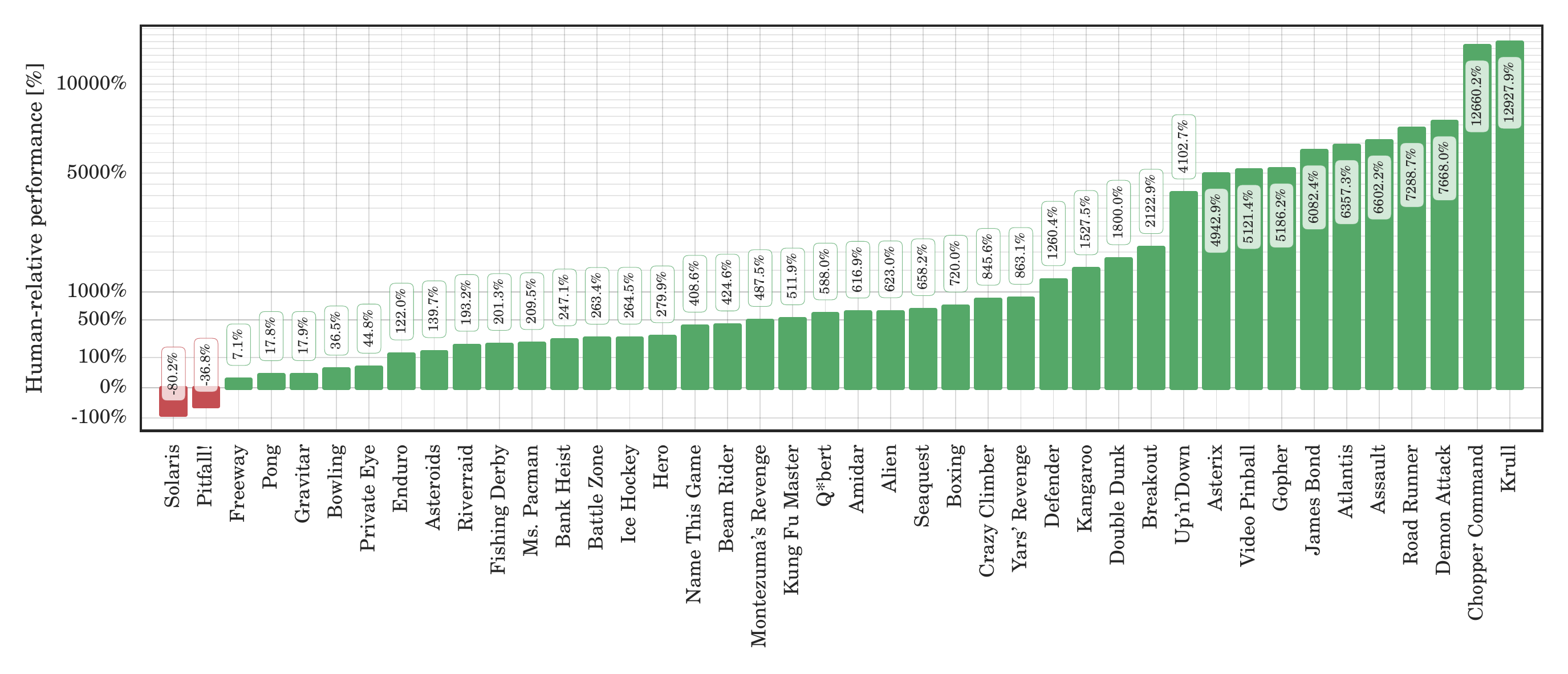}

    \caption{The human-relative score of Ape-X DQfD (deeper) using the no-ops
    starts regime. The score is computed as
    $\frac{\text{alg. score} - \text{avg. human score}}{
    \text{avg. human score} - \text{random score}} \times 100$.}
    \label{fig:relative_perf_noop}
\end{figure}

\FloatBarrier
\newpage

\begin{table*}[!h]
\tiny
\begin{tabularx}{\linewidth}{r?R-R-R?R-R?R-R-R}
\toprule
Game& Rainbow & DQfD & Ape-X DQN & Ape-X DQfD & Ape-X DQfD (deeper) & Random & Avg. Human & Expert \\
\midrule
Alien &6022.9 & -- & \textbf{ 17731.5 } & 1025.5 & 6983.4 & -- & 6371.3 & --\\
Amidar &202.8 & -- & 1047.3 & 310.5 & 1177.5 & -- & \textbf{ 1540.4 } & --\\
Assault &14491.7 & -- & 24404.6 & 23384.3 & \textbf{ 34716.5 } & -- & 628.9 & --\\
Asterix &280114.0 & -- & 283179.5 & \textbf{ 327929.0 } & 297533.8 & -- & 7536.0 & --\\
Asteroids &2249.4 & -- & \textbf{ 117303.4 } & 95066.6 & 95170.9 & -- & 36517.3 & --\\
Atlantis &814684.0 & -- & 918714.5 & 912443.0 & \textbf{ 1020311.0 } & -- & 26575.0 & --\\
Bank  Heist &826.0 & -- & 1200.8 & 1695.9 & \textbf{ 2020.5 } & -- & 644.5 & --\\
Battle Zone &52040.0 & -- & \textbf{ 92275.0 } & 42150.0 & 74410.0 & -- & 33030.0 & --\\
Beam Rider &21768.5 & -- & 72233.7 & 46454.5 & \textbf{ 82997.1 } & -- & 14961.0 & --\\
Bowling &39.4 & -- & 30.2 & \textbf{ 178.3 } & 174.4 & -- & 146.5 & --\\
Boxing &54.9 & -- & \textbf{ 80.9 } & 64.5 & 69.7 & -- & 9.6 & --\\
Breakout &379.5 & -- & \textbf{ 756.5 } & 145.1 & 365.5 & -- & 27.9 & --\\
Chopper Command &10916.0 & -- & 576601.5 & 90152.5 & \textbf{ 681202.5 } & -- & 8930.0 & --\\
Crazy Climber &143962.0 & -- & \textbf{ 263953.5 } & 141468.0 & 196633.5 & -- & 32667.0 & --\\
Defender &47671.3 & -- & \textbf{ 399865.3 } & 37771.8 & 123734.8 & -- & 14296.0 & --\\
Demon Attack &109670.7 & -- & 133002.1 & 97458.8 & \textbf{ 142189.0 } & -- & 3442.8 & --\\
Double Dunk &-0.6 & -- & \textbf{ 22.3 } & 20.5 & 21.8 & -- & -14.4 & --\\
Enduro &\textbf{ 2061.1 } & -- & 2042.4 & 1538.3 & 1754.9 & -- & 740.2 & --\\
Fishing Derby &22.6 & -- & 22.4 & \textbf{ 26.3 } & 24.0 & -- & 5.1 & --\\
Freeway &\textbf{ 29.1 } & -- & 29.0 & 23.8 & 26.8 & -- & 25.6 & --\\
Gopher &72595.7 & -- & \textbf{ 121168.2 } & 115654.7 & 115392.1 & -- & 2311.0 & --\\
Gravitar &567.5 & -- & 662.0 & 972.0 & 1021.8 & -- & \textbf{ 3116.0 } & --\\
Hero &50496.8 & -- & 26345.3 & 104942.1 & \textbf{ 107144.0 } & -- & 25839.4 & --\\
Ice Hockey &-0.7 & -- & \textbf{ 24.0 } & 3.3 & 18.4 & -- & 0.5 & --\\
James Bond &18142.3 & -- & \textbf{ 18992.3 } & 12041.0 & 15010.0 & -- & 368.5 & --\\
Kangaroo &10841.0 & -- & 577.5 & 25953.5 & \textbf{ 28616.0 } & -- & 2739.0 & --\\
Krull &6715.5 & -- & 8592.0 & 111496.1 & \textbf{ 122870.1 } & -- & 2109.1 & --\\
Kung Fu Master &28999.8 & -- & 72068.0 & 50421.5 & \textbf{ 102258.0 } & -- & 20786.8 & --\\
Montezuma's Revenge &154.0 & -- & 1079.0 & \textbf{ 22781.0 } & 22730.5 & -- & 4182.0 & --\\
Ms. Pacman &2570.2 & -- & 6135.4 & 1880.8 & 4007.4 & -- & \textbf{ 15375.0 } & --\\
Name This Game &11686.5 & -- & 23829.9 & 22874.6 & \textbf{ 29416.0 } & -- & 6796.0 & --\\
Pitfall! &-37.6 & -- & -273.3 & 3367.5 & 3208.7 & -- & \textbf{ 5998.9 } & --\\
Pong &\textbf{ 19.0 } & -- & 18.7 & 14.0 & 18.6 & -- & 15.5 & --\\
Private Eye &1704.4 & -- & 864.7 & 61895.1 & 54976.0 & -- & \textbf{ 64169.1 } & --\\
Q*bert &18397.6 & -- & \textbf{ 380152.1 } & 41419.6 & 51159.3 & -- & 12085.0 & --\\
Riverraid &15608.1 & -- & \textbf{ 49982.8 } & 18720.1 & 42288.9 & -- & 14382.2 & --\\
Road Runner &54261.0 & -- & 127111.5 & 486082.0 & \textbf{ 507490.0 } & -- & 6878.0 & --\\
Seaquest &19176.0 & -- & \textbf{ 377179.8 } & 15526.1 & 269480.0 & -- & 40425.8 & --\\
Solaris &2860.7 & -- & 3115.9 & 2235.6 & 1835.8 & -- & \textbf{ 11032.6 } & --\\
Up'n'Down &92640.6 & -- & \textbf{ 347912.2 } & 200709.3 & 298361.8 & -- & 9896.1 & --\\
Video Pinball &506817.2 & -- & \textbf{ 873988.5 } & 194845.0 & 832691.1 & -- & 15641.1 & --\\
Yars' Revenge &93007.9 & -- & 131701.1 & 82521.8 & \textbf{ 466181.8 } & -- & 47135.2 & --\\
\bottomrule
\end{tabularx}
\caption{Scores obtained by evaluating the best checkpoint for 200 episodes using the
human starts regime.}
\label{tbl:all_games_human}

\end{table*}

\begin{figure}[!h]
    \centering
    \includegraphics[width=\linewidth]{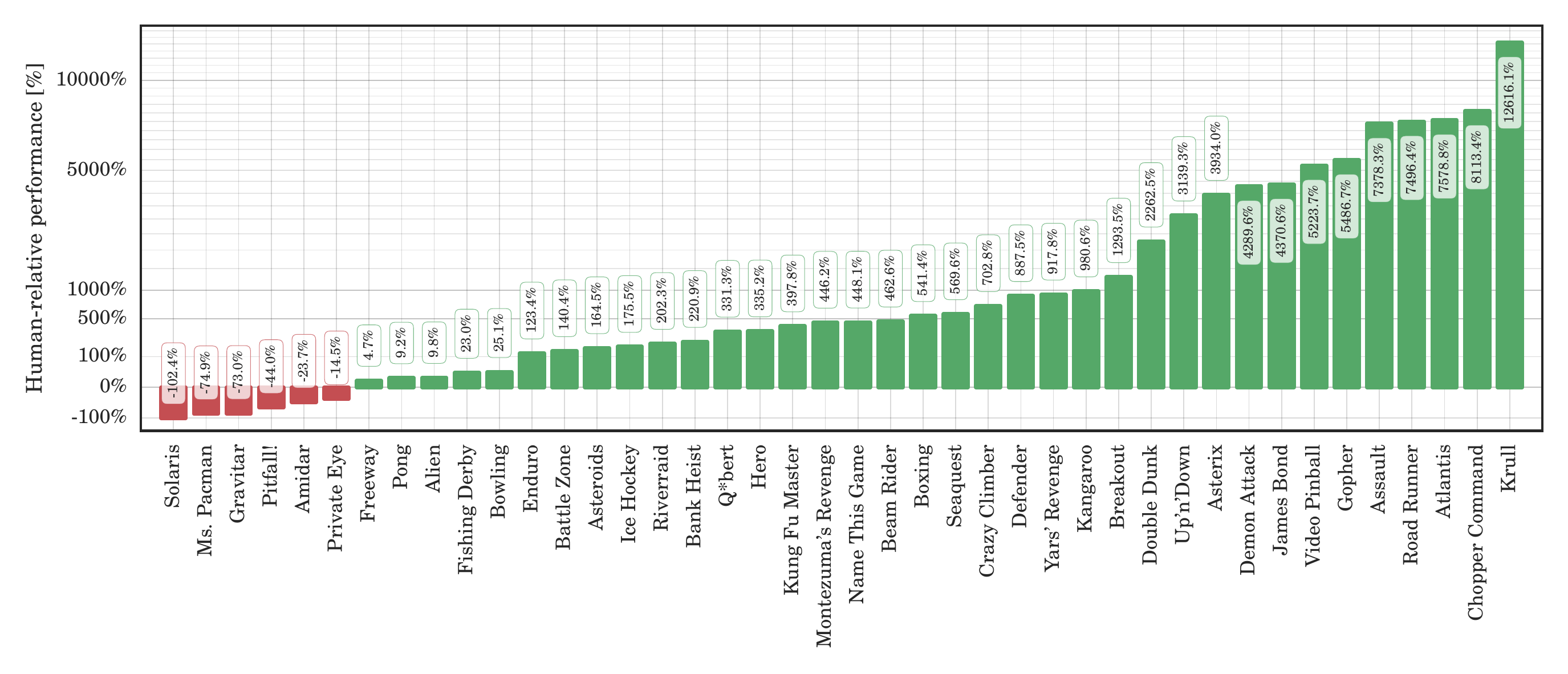}
    \caption{The human-relative score of Ape-X DQfD (deeper) using the human starts
    regime. The score is computed as
    $\frac{\text{alg. score} - \text{avg. human score}}{
    \text{avg. human score} - \text{random score}} \times 100$.}
    \label{fig:relative_perf_human}
\end{figure}

\FloatBarrier
\newpage

\section{Experimental setup \& hyper parameters}

\begin{table*}[!h]
\small
\begin{tabularx}{\linewidth}{r?R-R-R-R}
\toprule
Game & Min score & Max score & Number of transitions & Number of episodes\\
\midrule
Alien & 9690 & 29160 & 19133 & 5  \\ 
Amidar & 1353 & 2341 & 16790 & 5 \\ 
Assault & 1168 & 2274 & 13224 & 5  \\ 
Asterix & 4500 & 18100 & 9525 & 5  \\ 
Asteroids & 14170 & 18100 & 22801 & 5 \\ 
Atlantis & 10300 & 22400 & 17516 & 12  \\ 
Bank Heist & 900 & 7465 & 32389 & 7  \\
Battle Zone & 35000 & 60000 & 9075 & 5 \\ 
Beam Rider & 12594 & 19844 & 38665 & 4  \\ 
Bowling & 89 & 149 & 9991 & 5 \\ 
Boxing & 0 & 15 & 8438 & 5 \\ 
Breakout & 17 & 79 & 10475 & 9  \\ 
Chopper Command & 4700 & 11300 & 7710 & 5  \\ 
Crazy Climber & 30600 & 61600 & 18937 & 5  \\ 
Defender & 5150 & 18700 & 6421 & 5  \\ 
Demon Attack & 1800 & 6190 & 17409 & 5  \\ 
Double Dunk & -22 & -14 & 11855 & 5  \\ 
Enduro & 383 & 803 & 42058 & 5 \\ 
Fishing Derby & -10 & 20 & 6388 & 4  \\ 
Freeway & 30 & 32 & 10239 & 5  \\ 
Gopher & 2500 & 22520 & 38632 & 5  \\ 
Gravitar & 2950 & 13400 & 15377 & 5  \\ 
Hero & 35155 & 99320 & 32907 & 5  \\ 
Ice Hockey & -4 & 1 & 17585 & 5  \\ 
James Bond & 400 & 650 & 9050 & 5  \\ 
Kangaroo & 12400 & 36300 & 20984 & 5 \\ 
Krull & 8040 & 13730 & 32581 & 5  \\ 
Kung Fu Master & 8300 & 25920 & 12989 & 5  \\ 
Montezuma's Revenge & 32300 & 34900 & 17949 & 5  \\ 
Ms Pacman & 31781 & 55021 & 21896 & 3 \\
Name This Game & 11350 & 19380 & 43571 & 5  \\ 
Pitfall & 3662 & 47821 & 35347 & 5  \\ 
Pong & -12 & 0 & 17719 & 3 \\ 
Private Eye & 70375 & 74456 & 10899 & 5  \\ 
Q-Bert & 80700 & 99450 & 75472 & 5  \\ 
River Raid & 17240 & 39710 & 46233 & 5  \\ 
Road Runner & 8400 & 20200 & 5574 & 5 \\ 
Seaquest & 56510 & 101120 & 57453 & 7  \\ 
Solaris & 2840 & 17840 & 28552 & 6 \\ 
Up N Down & 6580 & 16080 & 10421 & 4\\ 
Video Pinball & 8409 & 32420 & 10051 & 5 \\ 
Yars' Revenge & 48361 & 83523 & 21334 & 4 \\ 
\bottomrule
\end{tabularx}
\caption{The table shows the performance of our expert player and the amount
of available demonstrations per game. Note that the total number of
episodes/trajectories is very low.}
\label{tbl:demos}

\end{table*}

\begin{figure}[!h]
    \centering
    \includegraphics[width=\linewidth]{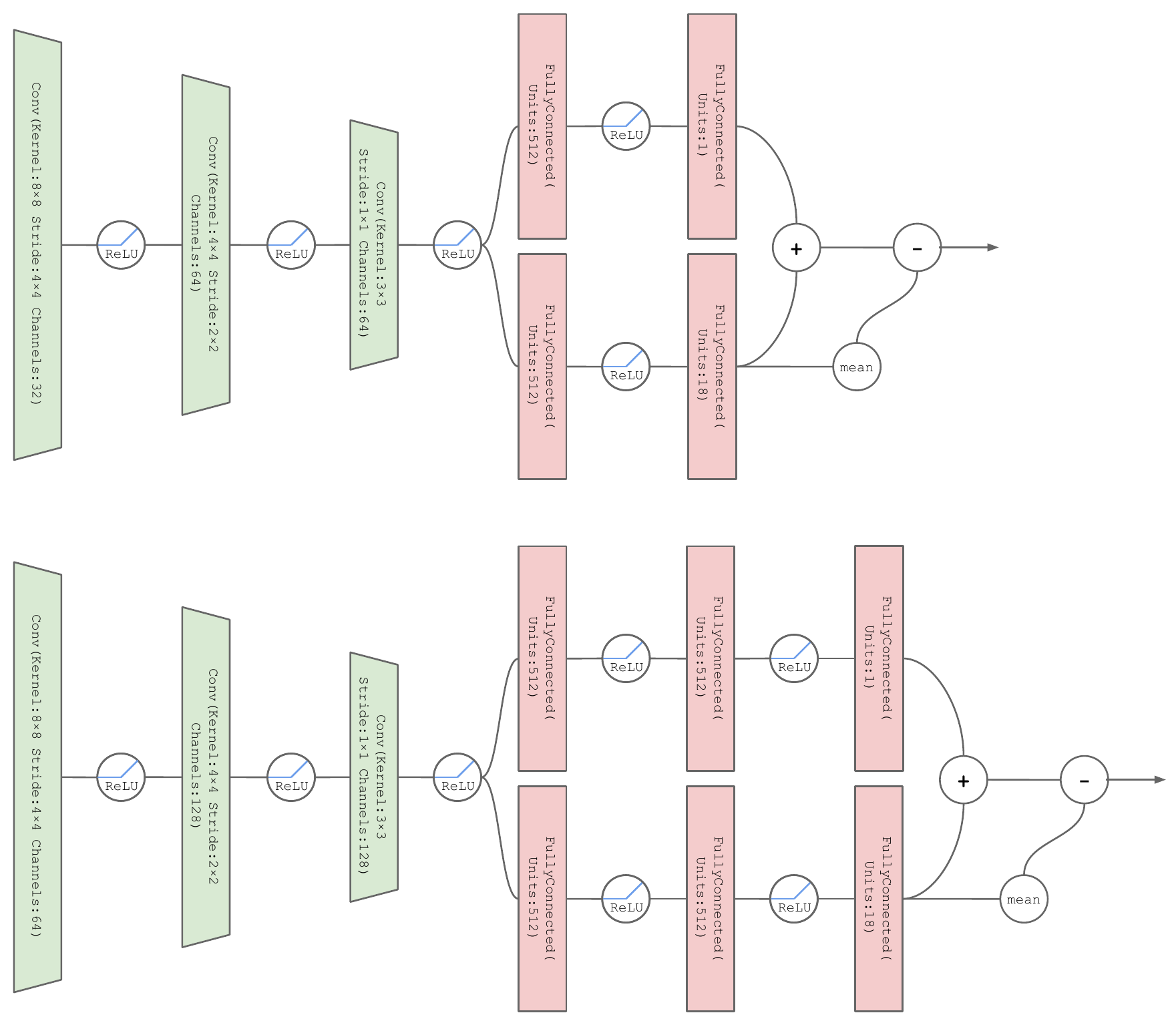}
    \caption{The two network architectures that we used. The upper one is the 
    standard dueling architecture of \citet{wang2016dueling} and the lower one
    is a slightly wider and deeper version.}
    \label{fig:networks}
\end{figure}

\begin{table*}[!h]
\small
\begin{tabularx}{\linewidth}{r?l-r}
\toprule
Parameter & Comment & Value \\
\midrule
\midrule
\multicolumn{3}{C}{Learner configuration} \\
\midrule
Batch size & & 256\\
Agent transitions per batch & & 192 \\
Expert transitions per batch & & 64 \\
Adam learning rate & & $5\cdot10^{-5}$\\
Adam regularizer & & $\frac{0.01}{256}$\\
Maximum gradient norm & We use \tt{tf.clip\_by\_global\_norm} & 40.0 \\
Target update period & Referred to as $T_\text{target}$ in the text & 2500 \\
Discount factor & Referred to as $\gamma$ in the text&  0.999 \\
Margin & Referred to as $\lambda$ in the text & $\sqrt{0.999}$ \\
\midrule
\multicolumn{3}{C}{Arcade Learning Environment (ALE) parameters} \\
\midrule
Use full Atari action set & & Yes \\
Repeat actions & & 4 \\
Expose lives & & No \\
\bottomrule
\end{tabularx}
\caption{The table shows all of our hyper parameters.}
\label{tbl:hypers}

\end{table*}

\FloatBarrier
\newpage
\section{Videos}
\label{sec:video}
%Due to the size of the videos, we uploaded them to an anonymized Google Drive account.
We provide videos comparing an agent episode to an expert episode on YouTube.

\textbf{\textsc{Montezuma's Revenge:}} \\
\begin{small}
\texttt{https://www.youtube.com/watch?v=-0xOdnoxAFo\&index=4\&list=PLnZpNNVLsMmOfqMwJLcpLpXKLr3yKZ8Ak}
\end{small}

\textbf{\textsc{Hero:}}\\
\begin{small}
\texttt{https://www.youtube.com/watch?v=T3uKDubwzhw\&list=PLnZpNNVLsMmOfqMwJLcpLpXKLr3yKZ8Ak\&index=2}
\end{small}

\textbf{\textsc{Bowling:}}\\
\begin{small}
\texttt{https://www.youtube.com/watch?v=67x1cFnSA\_c\&index=1\&list=PLnZpNNVLsMmOfqMwJLcpLpXKLr3yKZ8Ak}
\end{small}

\textbf{\textsc{Breakout:}}\\
\begin{small}
\texttt{https://www.youtube.com/watch?v=hr\_KNsQPe7U\&list=PLnZpNNVLsMmOfqMwJLcpLpXKLr3yKZ8Ak\&index=3}
\end{small}
\end{document}